\documentclass[11pt]{article}

\usepackage[preprint]{jmlr2e}

\usepackage{graphicx} %
\usepackage{fullpage}
\usepackage{amsmath}
\usepackage{amssymb}
\usepackage{mathtools}
\usepackage{soul}
\usepackage{stmaryrd}
\usepackage{enumitem}
\usepackage{comment}
\usepackage{rotating}
\usepackage[margin=1in,  headheight=0.5in, headsep=0.25in]{geometry}
\usepackage{xcolor}
\usepackage{hyperref}
\usepackage{cleveref}
\usepackage{wrapfig}
\usepackage{tabularx}
\usepackage{multirow} %
\usepackage{booktabs}  %
\usepackage[font=small,skip=4pt]{caption}
\usepackage{enumitem}
\usepackage{algorithm}
\usepackage{algpseudocode}
\usepackage{url}
\setlist{nolistsep}

\usepackage{booktabs}
\usepackage{siunitx}
\usepackage{listings}
\usepackage{xcolor}
\newcommand{\calM}{\mathcal{M}}
\newcommand{\calG}{\mathcal{G}}
\newcommand{\calB}{\mathcal{B}}

\begin{document}

\title{Beyond Muon: MUD (MomentUm Decorrelation)\\for Faster Transformer Training}

\author{\name Ben S. Southworth \email southworth@lanl.gov \\
       \addr Theoretical Division, Los Alamos National Laboratory, USA.\\
       \AND
       \name Stephen Thomas \email sjt223@lehigh.edu \\
       \addr Lehigh University, Computer Science}
\editor{}
\maketitle

\begin{abstract}
Orthogonalized-momentum optimizers such as Muon improve transformer training by approximately whitening/orthogonalizing matrix-valued momentum updates via a short polar-decomposition iteration. However, polar-factor approximations typically require multiple large matrix multiplications, and the resulting overhead can be substantial and hardware-dependent.
We introduce \textbf{MUD} (\textbf{Mo}ment\textbf{U}m \textbf{D}ecorrelation), a complementary whitening approach that replaces Muon’s polar update with a triangular (Cholesky-like) whitening surrogate inspired by classical Gram--Schmidt and Gauss--Seidel ideas. We show that row-orthonormal matrices are fixed points of the MUD map, relate the inner step to symmetric Gauss--Seidel preconditioning of the Gram matrix, and prove quadratic local convergence near the fixed point. In terms of time-to-perplexity, MUD yields consistent $10-50\%$ wall-clock improvements over tuned AdamW and Muon in time-to-perplexity, typically converging slightly slower per step than Muon but with substantially lower optimizer overhead -- relative to Muon, MUD improves peak tokens/s by roughly $1.3$--$2.6\times$ across most settings and up to nearly $3\times$ on GPT-2 large on an A100. We also demonstrate training a ESM-2 150M protein language model, where MUD matches Muon-level validation perplexity in significantly less wall-clock time.
\end{abstract}

\begin{keywords}
machine learning, Muon, transformer, optimization, LLM
\end{keywords}

\section{Introduction}
\label{sec:intro}

In training neural network models, the choice of optimizer defines both the evolution of the model in parameter space as well as the final loss and the trained model obtained. Faster optimizers reduce training cost and more effective optimizers improve service quality. For large language models (LLMs), despite substantial progress in architectures and scaling, the field has largely converged on adaptive first-order methods---especially AdamW---as a default for training Transformers \citep{kingma2014adam,loshchilov2017decoupled}. AdamW maintains per parameter first- and second-moment statistics and rescales each coordinate of the gradient, providing robustness across heterogeneous gradient magnitudes and easing hyperparameter tuning.

However, AdamW treats each parameter as an independent scalar. For a weight matrix $W\in\mathbb{R}^{n\times m}$ within the larger set of free parameters, it maintains $nm$ separate running averages and applies $nm$ separate coordinate-wise scaling factors. This component-wise treatment ignores the fact that $W$ represents a linear map $\mathbb{R}^m\rightarrow\mathbb{R}^n$ and that its gradient $\nabla_W \mathcal{L}$ (and momentum updates derived from it) can exhibit strong \emph{matrix-structured} correlations. In particular, the update geometry of a linear map is naturally described through its spectrum: correlations across rows or columns correspond to anisotropy in the row/column Gram matrices and to nontrivial singular-value structure. Coordinate-wise methods do not explicitly account for this structure.

Transformer language models are sequence models built from repeated blocks containing multi-head self-attention (MHSA) and position-wise feed-forward networks (MLPs/FFNs), interleaved with normalization and residual connections \citep{vaswani2017attention,radford2019language}. Given a sequence of token embeddings $x\in\mathbb{R}^{T\times d_{\text{model}}}$, each block applies a sequence of linear projections and nonlinearities to produce contextualized representations. Although implementations vary (pre-norm vs.\ post-norm, activation functions, fused attention kernels), the dominant trainable parameters remain concentrated in large \emph{linear weight matrices} that map between feature dimensions. In a typical decoder block with hidden size $d_{\text{model}}$ and MLP width $d_{\text{ff}}$, the heaviest parameter tensors are the 2D weights of:
\begin{itemize}
\item {Attention projections:} $W_Q,W_K,W_V \in \mathbb{R}^{d_{\text{model}}\times d_{\text{model}}}$ and $W_O \in \mathbb{R}^{d_{\text{model}}\times d_{\text{model}}}$ (up to fused QKV and implementation details).
\item {MLP projections:} $W_1 \in \mathbb{R}^{d_{\text{ff}}\times d_{\text{model}}}$ and $W_2 \in \mathbb{R}^{d_{\text{model}}\times d_{\text{ff}}}$ (sometimes with a third matrix in gated variants).
\end{itemize}
In GPT-style models, $d_{\text{ff}}$ is often $4d_{\text{model}}$, so $W_1$ and $W_2$ are typically rectangular with dimensions roughly $(4d_{\text{model}})\times d_{\text{model}}$ and $d_{\text{model}}\times (4d_{\text{model}})$. Consequently, matrix-shaped parameters dominate both the compute and memory footprints.

Viewing layer weights as coherent matrices has motivated many structured second-order and quasi-second-order methods. For example, K-FAC approximates the Fisher (or Gauss--Newton) curvature by a Kronecker factorization, producing updates that act on matrix-shaped parameters via left and right-multiplications by approximate inverses of input/output covariances \citep{martens2015kfac}. Shampoo and related methods precondition tensor gradients using inverse roots of second-moment matrices along multiple tensor modes, producing updates of the form $\Delta W \;\approx\; A^{-1/2}\, G\, B^{-1/2},$ where $G$ is the raw gradient and $A,B$ are accumulated second-moment factors \citep{gupta2018shampoo,vyas2024soap}. While powerful, these methods can be limited by the cost and complexity of maintaining matrix inverse-roots (or factorizations) at scale.

Recently, (approximate) {orthogonalized update} methods have gained significant attention as a cheaper alternative that takes advantage of the matrix structure without explicit estimation of curvature. The Muon optimizer applies a short Newton--Schulz iteration to approximately orthogonalize (or whiten) an evolving momentum/ trajectory matrix for each large hidden-layer weight, while updating embeddings, heads, biases, and normalization parameters with AdamW \citep{jordan2024muon,liu2025muon_scalable}. In linear-algebraic terms, Muon can be viewed as approximating the closest row-orthonormal matrix to each momentum update matrix in a Frobenius sense, replacing a raw direction $M$ with an approximately row-orthonormal direction $Q\approx (MM^T)^{-1/2}M$. This step of 'matrix-awareness' can improve time-to-quality in practice and has led to a growing body of research in related directions, including practical ingredients to make Muon stable and robust on a large scale \citep{liu2025muon_scalable,li2025normuon}, improved orthogonalization kernels to reduce Newton-Schulz accuracy, robustness, and overhead relative to methods like AdamW \citep{boissin2025turbomuon,amsel2025polarexpress,grishina2025cans}, and analysis on how (inexact) orthogonalization affects convergence and hyperparameter choice \citep{shulgin2025beyondideal,ma2026preconditioning_benefits_muon}. A broader perspective on matrix-whitening has also emerged, relating Muon to other matrix-aware optimizers and emphasizing that empirical gains are not solely explained by idealized spectral normalization, motivating alternative whitening constructions and hybrid variance-adaptive components \citep{frans2025really,frans2025splus,lau2025polargrad}.

Despite the strong time-to-quality improvements reported for Muon, its polar-factor approximation can contribute a substantial overhead computational cost.  \Cref{sec:mud} introduces \emph{MomentUm Decorrelation} (MUD), a new matrix-update operator that aims at the same whitening objective. While Muon approximates the symmetric polar factor through repeated Newton--Schulz iterations, MUD replaces this step with a \emph{triangular whitening} operator rooted in classical orthogonalization and stationary preconditioning ideas: it constructs a lower-triangular approximation to the Gram matrix and applies one or a small number of forward (Gauss--Seidel-like) triangular solves to decorrelate the momentum/lookahead matrix. The FLOP-count for a single triangular solve per step is approximately $12\times$ smaller than the original Muon algorithm and in most cases sufficient for good convergence. Row-orthonormal matrices are shown to be fixed points of the MUD map; the inner Gram approximation is connected to symmetric Gauss--Seidel preconditioning, and local quadratic convergence to a perfect whitening (identity Gram matrix) is proved (\Cref{th:conv}). \Cref{sec:experiments} evaluates MUD across multiple Transformer regimes and domains. In text LMs, MUD consistently improves wall-clock time-to-perplexity relative to tuned AdamW and Muon across three corpora (OpenWebText/NanoGPT, WikiText-103, FineWeb-Edu), three GPU platforms (A100, MI250, GH200), and GPT-2 small/medium/large architectures (up to 775M parameters), in many cases achieving a fixed validation perplexity 10--50\% faster. Complementing these time-to-quality results, throughput measurements (\Cref{sec:experiments:throughput}) show that MUD substantially reduces Muon’s per-step overhead, improving tokens/s by roughly $1.3$--$2.6\times$ in most settings and up to nearly $3\times$ on GPT-2 large on an A100. Finally, the method is shown to transfer beyond the natural language by training the ESM-2 150M protein language model on the \texttt{omg\_prot50} masked-language-modeling task, where MUD attains the validation perplexity of the Muon-level at much lower wall-clock times.

\section{MomentUm Orthogonalization based on Newton--Schulz (Muon)}
\label{sec:storyline}

\subsection{Descent methods}

Let $\mathcal{L}(\Theta)$ be a loss function with parameters $\Theta$ and let $\nabla_\Theta \mathcal{L}(\Theta_t)$ denote a mini-batch gradient at step $t$.
For a parameter {shaped by matrix} $W\in\mathbb{R}^{n\times m}$, the stochastic gradient is written as
\[
G_t \;\coloneqq\; \nabla_W \mathcal{L}(\Theta_t)\in\mathbb{R}^{n\times m},
\]
and reserve capital letters (e.g. $G_t,V_t,M_t,Q_t$) for matrix-valued quantities associated with such parameters.
For element wise or vector-shaped parameters bold notation is used; e.g., for $\boldsymbol{\theta}$ the corresponding gradient takes the form
\[
\boldsymbol{g}_t \;\coloneqq\; \nabla_{\boldsymbol{\theta}} \mathcal{L}(\Theta_t).
\]
For a matrix parameter $W$, the classical (heavy-ball) momentum maintains a matrix buffer $V_t$ by
$V_t = \beta V_{t-1} + G_t,$
with $\beta\in(0,1)$, typically $0.9$--$0.95$. The update is
$W_{t+1} = W_t - \eta V_t,$
where $\eta>0$ is the learning rate. The buffer $V_t$ is an exponential moving average of past gradients with characteristic memory length approximately $1/(1-\beta)$ steps.
A common Nesterov-style lookahead direction is
$M_t \;=\; G_t + \beta V_t,$
which replaces $V_t$ in the update $W_{t+1}=W_t-\eta M_t$.

AdamW maintains exponential moving averages of the first and second moments for (typically) vector- or element wise-shaped parameters $\boldsymbol{\theta}$:
\begin{align}
\boldsymbol{m}_t &= \beta_1 \boldsymbol{m}_{t-1} + (1-\beta_1)\boldsymbol{g}_t, \\
\boldsymbol{v}_t &= \beta_2 \boldsymbol{v}_{t-1} + (1-\beta_2)\boldsymbol{g}_t\odot \boldsymbol{g}_t,
\end{align}
where $\odot$ denotes multiplication by element. 
The bias-corrected moments are
$\hat{\boldsymbol{m}}_t=\boldsymbol{m}_t/(1-\beta_1^t)$ and
$\hat{\boldsymbol{v}}_t=\boldsymbol{v}_t/(1-\beta_2^t)$, resulting in the decoupled-weight-decay update
\begin{equation}
\label{eq:adamw}
\boldsymbol{\theta}_{t+1}
=
(1-\eta\lambda)\boldsymbol{\theta}_t
-\eta\,
\frac{\hat{\boldsymbol{m}}_t}{\sqrt{\hat{\boldsymbol{v}}_t}+\epsilon},
\end{equation}
where $\lambda$ is the weight decay coefficient and $\epsilon>0$ prevents division by zero.
AdamW is effective in handling coordinate-wise scale mismatches and sparse gradients, but for matrix parameters
$W\in\mathbb{R}^{n\times m}$ it treats the gradient as an element wise object and ignores its geometry as a linear map.

Orthogonalized momentum methods aim to exploit this matrix structure directly. For a gradient matrix $G\in\mathbb{R}^{n\times m}$, let
\begin{equation}
G = U\Sigma V^T = \sum_{i=1}^r \sigma_i \,\boldsymbol{u}_i \boldsymbol{v}_i^T,
\end{equation}
be a thin SVD, where $r=\min(n,m)$, $\sigma_1\ge \cdots \ge \sigma_r\ge 0$, and $\{\boldsymbol{u}_i\},\{\boldsymbol{v}_i\}$ are orthonormal singular vectors.
A gradient step $W\leftarrow W-\eta G$ advances by $\eta\sigma_i$ along each rank-one direction
$\boldsymbol{u}_i\boldsymbol{v}_i^T$.
When condition number $\kappa(G)=\sigma_1/\sigma_r$ is large, the update is highly anisotropic: large steps occur along dominant singular directions and small steps along trailing directions.
Empirically, gradient matrices in transformer training often exhibit a near-low rank structure, with a small number of singular values dominating the Frobenius norm.
For a transformer layer with input activations $X\in\mathbb{R}^{b\times d}$ and output gradients $\Delta\in\mathbb{R}^{b\times d}$, the gradient with respect to a weight matrix $W$ takes the form of
\[
\nabla_W\mathcal{L} \;=\; X^T\Delta,
\]
implying $\mathrm{rank}(\nabla_W\mathcal{L})\le b$ when $b<d$; even when $b\ge d$, the effective rank is often smaller due to correlations in the data.

\subsection{Muon}

Muon is designed for matrix-shaped parameters, typically large projection matrices in attention and MLP blocks.
Let $\Theta$ denote all model parameters and partition
\[
\Theta=\Theta_{\mathrm{mat}} \cup \Theta_{\mathrm{oth}},
\]
where $\Theta_{\mathrm{mat}}$ contains matrix-shaped weights (e.g., attention/MLP projection matrices, possibly including reshaped convolution kernels) and $\Theta_{\mathrm{oth}}$ contains all remaining parameters (embeddings, heads, biases, layer norm parameters).
The Muon updates are applied to $\Theta_{\mathrm{mat}}$, while the AdamW updates are applied to $\Theta_{\mathrm{oth}}$.

For a matrix parameter $W\in\mathbb{R}^{n\times m}$, let $G_t=\nabla_W \mathcal{L}(\Theta_t)$ be its stochastic gradient at step $t$.
Muon maintains a matrix momentum buffer $V_t$ and forms a Nesterov-style lookahead direction $M_t$:
\begin{equation}
\label{eq:muon-momentum}
V_t = \beta V_{t-1} + G_t,
\qquad
M_t = G_t + \beta V_t,
\end{equation}
where $\beta\in(0,1)$ (often $\beta\approx 0.95$).
Muon replaces the raw lookahead direction $M_t$ by an approximately orthonormal direction along the smaller dimension of $M_t$ for computational efficiency. 
Define $k=\min(n,m)$ and $d=\max(n,m)$; the convention transposes to work in $\mathbb{R}^{k\times d}$ when needed:
\begin{equation}
\label{eq:muon-transpose}
M \mapsfrom
\begin{cases}
M_t, & n\le m,\\
M_t^T, & n>m,
\end{cases}
\qquad \implies\qquad
M\in\mathbb{R}^{k\times d}.
\end{equation}
Muon aims to produce $Q_t\in\mathbb{R}^{k\times d}$ that represents the basis of $M$ with approximately orthonormal rows, $Q_tQ_t^T\approx I_k$  (and transposes back if needed to match the shape of $W$).

Muon-style methods formally target the closest row-orthonormal matrix to $M$ in the Frobenius norm under the Stiefel constraint $QQ^T=I_k$:
\begin{equation}
\label{eq:closest-orthogonal}
Q_\star \;\in\; \arg\min_{Q:\,QQ^T =I_k}\ \|M - Q\|_F,
\qquad Q\in\mathbb{R}^{k\times d},\ k\le d.
\end{equation}
When $M$ has the full row rank, the unique solution is the {polar factor}~\citep{higham2008functions}:
\begin{equation}
\label{eq:polar-factor}
Q_\star \;=\; (MM^T)^{-1/2}M.
\end{equation}
Equivalently, for a thin SVD $M=U\Sigma V^T$ with $U\in\mathbb{R}^{k\times k}$ and $V\in\mathbb{R}^{d\times k}$, one has $Q_\star=UV^T$.
By construction, $Q_\star Q_\star^T=I_k$, and among all row-orthonormal matrices, it is closest to $M$ in the Frobenius norm.

Muon avoids explicitly computing $(MM^T)^{-1/2}$ (or an SVD) by applying a small number of Newton--Schulz-like iterations that approximate the polar factor using dense matrix multiplications and element wise operations~\citep{jordan2024muon,liu2025muon_scalable}.
A standard implementation first normalizes $M$ (e.g., by Frobenius norm) so that the spectrum of $MM^T$ lies in a bounded range suitable for polynomial iterations. With
\[
X_0 \;=\; \frac{M}{\|M\|_F+\epsilon},
\]
Muon applies a fixed polynomial update
\begin{equation}
\label{eq:ns-poly}
X_{j+1} = aX_j + b(X_jX_j^T)X_j + c(X_jX_j^T)^2X_j,
\end{equation}
with coefficients $(a,b,c)=(3.4445,-4.7750,2.0315)$ and a fixed iteration count (typically five).
Each step forms the $k\times k$ Gram matrix $X_jX_j^T$ and multiplies it against $X_j$.
For $M=U\Sigma V^T$, the output can be written as $X_J=U\,\varphi_J(\Sigma)\,V^T$, where $\varphi_J$ is a scalar polynomial applied to the singular values; the singular vectors are preserved, while the singular values are pushed towards one.

Together, Muon: (i) constructs a matrix update direction $M$ (typically a momentum / lookahead direction), (ii) approximates $Q\approx (MM^T)^{-1/2}M$ through \eqref{eq:ns-poly}, and (iii) uses $Q$ as the update direction for matrix-shaped parameters.
Muon is typically used with decoupled weight decay:
\begin{equation}
\label{eq:muon-update}
W \leftarrow (1-\eta\lambda)W - \eta\, s(W)\, Q_t,
\end{equation}
where $\eta$ is the learning rate, $\lambda$ is the weight decay and $s(W)$ is a shape-dependent scaling (e.g., a function of $n$ and $m$) motivated by Muon scaling analyses~\citep{liu2025muon_scalable}.
The experiments use the scaling and weight decay from \citet{liu2025muon_scalable} for Muon and MUD.

\begin{algorithm}[t]
\caption{Muon update for a matrix-shaped parameter (AdamW is applied to $\Theta_{\mathrm{oth}}$ separately)}
\label{alg:muon}
\begin{algorithmic}[1]
\Require Matrix parameter $W\in\mathbb{R}^{n\times m}$, gradient $G_t$, momentum buffer $V_{t-1}$, lr $\eta$, weight decay $\lambda$, momentum $\beta$, NS steps $S$, scaling $s(W)$, $\varepsilon>0$
\State $V_t \leftarrow \beta V_{t-1} + G_t$ \Comment{matrix momentum}
\State $M_t \leftarrow G_t + \beta V_t$ \Comment{Nesterov lookahead}
\If{$n>m$} $M_t \leftarrow M_t^T$ \EndIf \Comment{work in $\mathbb{R}^{k\times d}$ with $k=\min(n,m)$}
\State $X \leftarrow M_t/(\|M_t\|_F+\varepsilon)$ \Comment{Frobenius normalization for NS stability}
\For{$j=1,\dots,S$}
    \State $A \leftarrow XX^T$ \Comment{$A\in\mathbb{R}^{k\times k}$}
    \State $X \leftarrow aX + b(AX) + c(A(AX))$ \Comment{implements \eqref{eq:ns-poly}}
\EndFor
\State $Q_t \leftarrow X$
\If{$n>m$} $Q_t \leftarrow Q_t^T$ \EndIf \Comment{undo transpose}
\State $W \leftarrow (1-\eta\lambda)W - \eta\, s(W)\, Q_t$ \Comment{decoupled weight decay + scaled step}
\State \Return $W, V_t$
\end{algorithmic}
\end{algorithm}

\section{Approximate MomentUm Decorrelation (MUD)}\label{sec:mud}

\subsection{Orthogonalization and decorrelation}

Orthogonalization is a classical theme in numerical linear algebra. Given a collection of vectors arranged as columns of a matrix $A\in\mathbb{R}^{d\times k}$ (with $d\ge k$), Gram--Schmidt produces a QR factorization $A=QR$ with $Q^T  Q=I_k$ and $R$ upper triangular. Classical Gram--Schmidt (CGS) and modified Gram--Schmidt (MGS) differ in the order in which projections are applied; in finite precision, MGS is typically more stable and is a standard workhorse for constructing orthonormal bases~\citep{ruhe1983gram_schmidt,golub2013matrix}. A second classical method is triangular orthogonalization via the Gram matrix. Defining $G=A^TA$ (the normal equation matrix), one may compute
\begin{equation}
L = \mathrm{chol}(G), 
\qquad
Q = A L^{-T},
\label{eq:cholqr}
\end{equation}
where chol$(G)$ denotes the Cholesky factorization $G=LL^T$, returning the lower triangular factor $L$. This yields $Q^TQ=I_k$ in exact arithmetic when $A$ has the full column rank, and is typically called \emph{CholeskyQR}. This approach is computationally appealing because it reduces orthogonalization to forming a Gram matrix, a Cholesky factorization, and applying a triangular solution. However, \eqref{eq:cholqr} can be numerically delicate. The formation of $A^T  A$ squares the condition number, and the resulting Cholesky factorization may be unstable or fail when $A$ is ill-conditioned or nearly rank-deficient. This classical stability consideration is directly relevant to optimizer update matrices, which can be highly correlated and effectively low rank early in training. Stabilized variants, e.g., reorthogonalization or CholeskyQR2, have been studied ~\citep{golub2013matrix,fukaya2020shifted}. 

The Muon and triangular orthogonalization are both realizations of \emph{whitening} or \emph{decorrelation} \citep{kessy2018optimal}. Given a matrix $M\in\mathbb{R}^{k\times d}$, define the row Gram matrix
\begin{equation}
\calG = MM^T  \in \mathbb{R}^{k\times k}.
\end{equation}
Note, one can also think of this as a form of covariance matrix over matrix variable $M$ by subtracting row means and rescaling, but for the purposes of this analysis, the Gram captures the correlations relevant to the update geometry. Consider a linear transform on the left $W\in\mathbb{R}^{k\times k}$, and define $Q = WM$. The matrix $Q$ is said to be \emph{whitened} in rows if $QQ^T  = I_k$, i.e.,
\begin{equation}
WM(M^T  W^T ) = I_k
\quad \Longleftrightarrow \quad
W \calG W^T  = I_k.
\label{eq:whitening-condition}
\end{equation}
A weaker goal is \emph{decorrelation}, which requires $W \calG W^T$ to be diagonal (not necessarily identity). A whitening operator is not unique, and there are infinitely many valid whitening transforms. Two canonical choices are:
\begin{enumerate}
    \item \ul{Symmetric (polar) whitening;} Taking $W = \calG^{-1/2}$ yields
    \begin{equation}
        Q_{\mathrm{polar}} = \calG^{-1/2}M,
    \end{equation}
    and $Q_{\mathrm{polar}}Q_{\mathrm{polar}}^T  = I_k$ in exact arithmetic. This is precisely the polar-factor solution \eqref{eq:polar-factor} and corresponds to the Frobenius-closest row-orthonormal approximation.

    \item \ul{Cholesky (triangular) whitening:} If $\calG = L_{\mathrm{chol}}L_{\mathrm{chol}}^T $ is a Cholesky factorization with $L_{\mathrm{chol}}$ lower triangular, then taking $W=L_{\mathrm{chol}}^{-1}$ yields
    \begin{equation}\label{eq:chol-whiten}
    Q_{\mathrm{chol}} = L_{\mathrm{chol}}^{-1}M,
    \qquad
    Q_{\mathrm{chol}}Q_{\mathrm{chol}}^T  = I_k
    \quad \text{(exact arithmetic)}.
    \end{equation}
    This is the triangular analog of whitening and is the basis for CholeskyQR-like orthogonalization.
\end{enumerate}
These two forms achieve the same whitening condition \eqref{eq:whitening-condition} but differ structurally: $\calG^{-1/2}$ is symmetric and rotation-invariant, whereas $L_{\mathrm{chol}}^{-1}$ is triangular and depends on ordering. The analytic form of the polar whitening is built from an SVD, which is more expensive than a Cholesky decomposition, but also stable and robust to ill-conditioned or singular matrices. Either way, in practice computing full decompositions tends to be prohibitively expensive, at least to be competitive with other state of the art methods. 

\subsection{MUD: a fast triangular approximation to whitening}

Recall that Muon approximates symmetric decorrelation via five Newton--Schulz iterations. Numerical evidence demonstrates that Muon matrices are often far from orthogonal, and repeated iterations to achieve true orthogonality yield minimal improvement in convergence of the outer machine learning optimization, e.g. \cite{shulgin2025beyondideal,liu2025muon_scalable}. Motivated by this and the structural simplicity of the triangular whitening, as well as its connection to Gram-Schmidt orthogonalization methods, the present work proposes cheap decorrelation algorithms based on triangular whitening. The objective is to obtain most of the practical benefit of triangular whitening while avoiding the numerical fragility and overhead cost of exact Cholesky factorization on dense, potentially ill-conditioned Gram matrices arising during training. To this end, the analysis adopts a perspective from classical numerical analysis: Gram--Schmidt variants can be viewed as stationary iterations (Jacobi/Gauss-Seidel) applied to normal-equation-like structures, and MGS in particular has a Gauss--Seidel flavor due to its sequential use of the most recently computed orthonormal vectors~\citep{ruhe1983gram_schmidt,thomas2022igs_gmres}. Indeed, recent work exploits this viewpoint to design iterated GS procedures with reduced synchronization in Krylov methods~\citep{thomas2022igs_gmres}.

The present setting does not maintain a growing orthonormal basis. Instead, at each optimization step, a matrix update direction $M$ is formed and a \emph{single} sequential linear preconditioner derived from the current Gram matrix is applied. Rather than using the true Cholesky factor as in \eqref{eq:chol-whiten}, $L_{\mathrm{chol}}$ is replaced with the lower triangular part of the Gram matrix,
\begin{equation}
    L \coloneqq \mathrm{tril}(\calG) \qquad\implies\qquad Q = L^{-1}M.
\end{equation}
This choice is motivated by the Gauss--Seidel viewpoint: lower-triangular structure corresponds to a forward sequential preconditioner that incorporates ``past'' correlations while ignoring ``future'' ones. This is wrapped with a row-normalization for stability, yielding an approximate MomentUm Decorrelation (MUD) nonlinear fixed-point iteration. This one-pass variant is denoted {MUD1}, and is applied to the matrix momentum/projection directions $M$ used in the optimizer. Repetition of MUD iterations as a fixed-point operator can further reduce correlation, denoted {MUDp} for $p$ passes. The general algorithm {MUDp} is provided in \Cref{alg:mud}. The following subsection proves local quadratic convergence of the MUD iteration to a decorrelated Gram matrix; \Cref{sec:experiments} demonstrates empirically that this one-pass surrogate captures most of the practical benefit of decorrelation while substantially reducing the computational cost relative to multi-sweep Newton--Schulz iterations. 

\begin{algorithm}[t]
\caption{MUDp: triangular Gram preconditioning for row-orthonormalization}
\label{alg:mud}
\begin{algorithmic}[1]
\Require Matrix update direction $M\in\mathbb{R}^{n\times m}$, passes $p\ge 1$, stability $\varepsilon>0$
\Ensure $Q$ with $QQ^T\approx I_k$ along the smaller dimension (transpose back if needed)
\State $k \gets \min(n,m)$,\quad $d \gets \max(n,m)$
\If{$n>m$} $M \gets M^T$ \EndIf \Comment{now $M\in\mathbb{R}^{k\times d}$}
\For{$j=1,\dots,p$}
    \State $\boldsymbol{r} \gets \big(\|M_{1,:}\|_2,\dots,\|M_{k,:}\|_2\big)$ \Comment{row norms of $M$, $O(kd)$ reduction}
    \State $Q \gets \mathrm{diag}\!\big((\boldsymbol{r}+\varepsilon)^{-1}\big)\,M$ \Comment{row-normalize}
    \State $\calG \gets QQ^T$ \Comment{Gram matrix in $\mathbb{R}^{k\times k}$ (GEMM)}
    \State $T \gets \mathrm{tril}(\calG)$ \Comment{lower-triangular Gram approximation}
    \State $Q \gets T^{-1}Q$ \Comment{forward triangular solve (TRSM)}
    \State $\boldsymbol{r} \gets \big(\|Q_{1,:}\|_2,\dots,\|Q_{k,:}\|_2\big)$ \Comment{row norms, $O(kd)$ reduction}
    \State $Q \gets \mathrm{diag}\!\big((\boldsymbol{r}+\varepsilon)^{-1}\big)\,Q$ \Comment{row-normalize again}
\EndFor
\If{$n>m$} $Q \gets Q^T$ \EndIf
\State \Return $Q$
\end{algorithmic}
\end{algorithm}

\textbf{Complexity.}
Let $k=\min(n,m)$ and $d=\max(n,m)$ after the Muon/MUD transpose convention so that $M\in\mathbb{R}^{k\times d}$ and the Gram matrices are $k\times k$.
A single Gram formation $XX^T$ costs approximately $2k^2d$ FLOPs (one $k\times d$ by $d\times k$ GEMM).
In Muon, each Newton--Schulz {quintic} step is
\[
X_{j+1}=aX_j+b(AX_j)+c\big(A(AX_j)\big),\qquad A=X_jX_j^T,
\]
requires forming $A$ (one Gram) and applying $A$ to a $k\times d$ matrix twice (two $k\times k$ by $k\times d$ GEMMs), for a total of roughly
$2k^2d + 2(2k^2d) \approx 6k^2d$ FLOPs per iteration.
With $S=5$ iterations, Muon performs five Gram formations and ten Gram applications, i.e., about $30k^2d$ FLOPs in the dominant dense-linear-algebra kernels.

In contrast, {MUD1} performs a Gram formation $G=QQ^T$ (about $2k^2d$ FLOPs), followed by a triangular solve $Q\leftarrow T^{-1}Q$ with $T=\mathrm{tril}(G)$.
Counting multiply-adds, a TRSM applying a $k\times k$ triangular factor to a $k\times d$ right-hand side costs approximately $\tfrac{1}{2}k^2d$ FLOPs, plus two row-norm reductions at cost $\Theta(kd)$.
Thus, {MUD1} has dominant cost $\approx 2.5k^2d$ FLOPs, and {MUD2} costs $\approx 5k^2d$ FLOPs.
Compared to Muon5, this corresponds to a reduction in operation-count of roughly $12\times$ for {MUD1} (and $6\times$ for {MUD2}), which is summarized in \Cref{tab:muon-mud-cost}.
In wall-clock terms on GPUs, realized speedups are typically smaller than FLOP ratios because triangular solves (TRSM) generally achieve lower throughput than highly optimized GEMM kernels; nonetheless, MUD reduces optimizer overhead substantially while retaining most of the practical benefit of decorrelating matrix update directions.

\begin{table}[!h]
\centering
\begin{tabular}{lcccc}
\toprule
Method & \#Grams & \#Applies & \#TRSM & Leading-order FLOPs \\
\midrule
Muon5 (quintic) & $5$ & $10$ & $0$ & $\approx 30\,k^2d$ \\
MUD1            & $1$ & $0$  & $1$ & $\approx 2.5\,k^2d$ \\
MUD2            & $2$ & $0$  & $2$ & $\approx 5\,k^2d$ \\
\bottomrule
\end{tabular}
\caption{Dominant dense-linear-algebra cost per parameter update (after transposing so $M\in\mathbb{R}^{k\times d}$, $k=\min(n,m)$, $d=\max(n,m)$). ``Gram'' denotes $X X^T$, $\approx 2k^2d$ FLOPs; ``Apply'' denotes multiplying a $k\times k$ matrix into a $k\times d$ matrix, $\approx 2k^2d$ FLOPs; TRSM denotes solving a $k\times k$ triangular system with $k\times d$ right-hand side, $\approx \tfrac{1}{2}k^2d$ FLOPs (counting multiply-adds). Constants ignore lower-order reductions (row norms) and scalar operations.}
\label{tab:muon-mud-cost}
\end{table}

\subsection{Analysis}

This subsection provides a local convergence analysis of the MUD fixed-point iteration, carried out in the Gram matrix space where the objective is a diagonally decorrelated Gram matrix. Define the correlation normalization as the symmetric scaling of $\calG$ to unit diagonal,
\begin{equation}\label{eq:corr-norm}
    \mathrm{Corr}(\calG) = D^{-1/2}\calG D^{-1/2},
\end{equation}
for $D \coloneqq\mathrm{diag}(\calG)$. Then MUD corresponds to the nonlinear fixed-point map
\begin{equation}
    \calG \mapsfrom \calM{\calG} \coloneqq \mathrm{Corr}\Big( \mathrm{tril}(\mathrm{Corr}(G))^{-1}\mathrm{Corr}(G)\mathrm{tril}(\mathrm{Corr}(G))^{-T}\Big).
\end{equation}
For convenience, combine the normalization of the first row of $M$ with a preprocessing step, such that the initial Gram matrix takes the form $\mathrm{diag}(\calG_0) = \mathrm{diag}(MM^T) = I$. Then $\mathrm{Corr}(G_0) = G_0$, and all subsequent Gram matrices remain normalized to the unit diagonal. The MUD fixed-point iterations in Gram matrix space then take the simplified form
\begin{equation}\label{eq:gram-update}
    \calG_{t+1} = \mathrm{Corr}\Big( \mathrm{tril}(\calG_t)^{-1}\calG_t\mathrm{tril}(\calG_t)^{-T}\Big).
\end{equation}

\Cref{prop:fp} establishes that any row-orthonormal matrix is a fixed point of the MUD iteration. The convergence to a true whitening operator implies that the eigenvalues of the Gram matrix converge towards unity, the condition $\calG = I$ holding precisely when all eigenvalues equal one. \Cref{prop:sgs_equivalence} relates MUD to classical symmetric Gauss-Seidel fixed-point iterations. Local quadratic convergence is then proved in \Cref{th:conv}, and a spectral corollary is derived in \Cref{cor:eigs_cluster}.

\begin{proposition}[Orthonormal fixed point]\label{prop:fp}
    The Gram matrix induced by any orthonormal row matrix $Q$, where $\calG = QQ^T = I_k$, is a fixed point of $\calM$. 
\end{proposition}
\begin{proof}
For $\calG= I$, observe that $\cal{G} = \mathrm{Corr}(\calG) = \mathrm{tril}(\calG) = I$. It follows immediately that $\calM(\calG) = \calM(I) = I$. 
\end{proof}

\begin{proposition}[Inner step as symmetric Gauss Seidel preconditioning]
\label{prop:sgs_equivalence}
Let $\calG\in\mathbb{R}^{k\times k}$ be SPD with unit diagonal, and write $\calG = I + \hat{L} + \hat{L}^T,$ where $L$ is strictly lower triangular. Define the lower-triangular factor $T := I+L$ and the Gram transform in the inner-step of MUD $\calB \coloneqq T^{-1}\calG_0T^{-T}$. Then $\calB$ is SPD. Moreover, the symmetric Gauss-Seidel (SGS) preconditioner for $\calG$ takes the form $M := T T^T$, and
\[
\sigma(\calB) = \sigma(M^{-1}\calG),
\]
that is, the eigenvalues of the Gram update in the inner-step $\calB$ \eqref{eq:calB} coincide with those of a Gram matrix preconditioned with SGS.
\end{proposition}
\begin{proof}
Because$\calB$ is a congruence transform of $G$ by the nonsingular matrix $T^{-1}$, it is SPD. For the spectral equivalence, similarity holds by
\[
\calB = T^{-1}\calG T^{-T} \sim T^{-T}T^{-1}\calG = (T T^T)^{-1}\calG =
M^{-1}\calG.
\]
\end{proof}

\begin{theorem}[Local quadratic convergence]\label{th:conv}
    Fix a vector-induced operator norm $\|\cdot\|$. There exist constants $\epsilon> 0$, $C>0$ such that if $\calG_0 = MM^T =  I + E_0$ with $\mathrm{diag}(\calG_0) = I$ and $\|E_0\| \leq \epsilon$, then the MUD iterates satisfy $\|E_{t+1}\| \leq C\|E_t\|^2$ for all $t$, and hence
    \begin{equation}
        \calG_{t} = I + \mathcal{O}(\|E_0\|^{2^{t}}). 
    \end{equation}
\end{theorem}
\begin{proof}
Let $\calG_0 \coloneqq  I + E_0$ denote the initial Gram matrix, where $E_0=E_0^T$ and diag$(E_0) = \mathbf{0}$. Let $\hat{L}$ denote the strictly lower triangular part of $\calG_0$, wherein $E_0 = \hat{L} + \hat{L}^T$, and $\|\hat{L}\| \leq c_E\|E_0\|$ for some $c_E$ depending on the chosen  (e.g., $c_E=1$ for the $\ell^1$- or $\ell^\infty$-norms). Define $T = \mathrm{tril}(\calG_0) = I + \hat{L}$ and
\begin{equation}
    \calB \coloneqq \mathrm{tril}(\calG_0)^{-1}\calG_0\mathrm{tril}(\calG_0)^{-T}
    = T^{-1}\calG_0T^{-T}.\label{eq:calB}
\end{equation}
Also note the identity $TT^T = \calG_0 + \hat{L}\hat{L}^T$. Plugging this into \eqref{eq:calB} yields 
\begin{align*}
    \calB = T^{-1}(TT^T - \hat{L}\hat{L}^T)T^{-T} = I - T^{-1}\hat{L}\hat{L}^TT^{-T}.
\end{align*}
Let $\|\cdot\|$ be a submultiplicative operator norm. The perturbation from identity then is bounded as
\begin{align*}
    \|I - \calB\| &=  \|I - (I + \hat{L})^{-1}(I + \hat{L} + \hat{L}^T)(I + \hat{L}^T)^{-1}\|
    \leq \|T^{-1}\hat{L}\hat{L}^TT^{-T}\| \\
    &\qquad \leq \|T^{-1}\hat{L}\|^2 
     = \Big\|\sum_{\ell=1}^{k-1} (-1)^{\ell-1}\hat{L}^\ell \Big\|^2 
    \leq \Big( \sum_{\ell=1}^{k-1} \|\hat{L}\|^\ell \Big)^2 
    < \left( \frac{\|\hat{L}\|}{1 - \|\hat{L}\|} \right)^2 \qquad \textnormal{for }\|\hat{L}\|<1.
\end{align*}

Now, let $D = \mathrm{diag}(\calB)$. Then
\[ 
    \|\mathrm{Corr}(\calB) - I\| = \|D^{-1/2}(\calB - D)D^{-1/2}\| 
    \leq \|D^{-1/2}\|^2\|\calB - D\| \leq \|D^{-1}\|\left( \|\calB - I\| + \|D - I\|\right).
\]
For any vector-induced operator norm, $\|D-I\| = \max_i |D_{ii}-1| = \max_i|\calB_{ii}-1| \leq \|\calB - I\|$, and thus $(\|\calB - I\| + \|D - I\|) \leq 2\|\calB-I\|$. For SPD $\calG$, the matrix $\calB = T^{-1}\calG T^{-T}$ is also SPD and thus $\calB_{ii} > 0$ for all $i$. In addition, $\calB = I - T^{-1}\hat{L}\hat{L}^TT^{-T}$ implies that for all $i$
\[
    \calB_{ii} = 1 - \mathrm{diag}(T^{-1}\hat{L}\hat{L}^TT^{-T})_{ii} \geq 1 - \|T^{-1}\hat{L}\hat{L}^TT^{-T}\|  \geq 1 -  \left( \frac{\|\hat{L}\|}{1 - \|\hat{L}\|} \right)^2.
\]
Then
\[
    \|D^{-1}\| = \frac{1}{\min_i \calB_{ii}} \leq \frac{1}{1 - \left( \frac{\|\hat{L}\|}{1 - \|\hat{L}\|} \right)^2}.
\]
Combining the above results, for $\calG_1 \coloneqq I + E_1$ one obtains
\[
    \|E_1\| = \|\mathrm{Corr}(\calB) - I\| \leq 2\|D^{-1}\|\|I-\calB\| < \frac{2\left( \frac{\|\hat{L}\|}{1 - \|\hat{L}\|} \right)^2}{1 - \left( \frac{\|\hat{L}\|}{1 - \|\hat{L}\|} \right)^2} = \frac{2\|\hat{L}\|^2}{1-2\|\hat{L}\|}.
\]
Choose $\epsilon$ such that $\|\hat{L}\| \leq c_E\|E_0\|\leq 1/3$. Then
\begin{equation}
    \|E_1\| < \frac{2c_E^2\|E_0\|^2}{1-2c_E\|E_0\|} \leq 6c_E^2\|E_0\|^2.
\end{equation}
It follows that $\|E_{t}\| \leq C\|E_{t-1}\|^2 \implies \|E_{t}\| \leq C^{2^{t}-1}\|E_0\|^{2^t}$. Writing in terms of $\calG_t = I + E_t$ completes the proof.
\end{proof}

\begin{corollary}[Local eigenvalue clustering and condition number improvement]
\label{cor:eigs_cluster}
Let $\calG_1$ denote the Gram matrix after one MUD pass, and suppose that the assumptions of \Cref{th:conv} hold, so that $\|\calG_1 - I\|_2 \le C \|\calG_0 - I\|_2^2.$
Let $\varepsilon := \|\calG_0 - I\|_2$. Then every eigenvalue of $\calG_1$ lies in the interval
\[
\sigma(\calG_1) \subset \bigl[\,1 - C\varepsilon^2,\; 1 + C\varepsilon^2\,\bigr],
\]
and in particular the condition number satisfies
\[
\kappa_2(\calG_1) \le \frac{1 + C\varepsilon^2}{1 - C\varepsilon^2}
\qquad \textnormal{for } C\varepsilon^2 < 1.
\]
\end{corollary}

\begin{proof}
By Weyl's inequality,
\[
|\lambda_i(\calG_1) - 1| \le \|\calG_1 - I\|_2 \le C\varepsilon^2
\]
for all $i$, which implies the spectral enclosure. The condition number bound follows immediately from $\lambda_{\max}(\calG_1) \le 1 + C\varepsilon^2$ and $\lambda_{\min}(\calG_1) \ge 1 - C\varepsilon^2$.
\end{proof}

\subsection{Algorithmic discussion}
\label{sec:efficiency-attempts}

This section concludes with a discussion of algorithmic choices and a summary of several efficiency-oriented variants that were explored but not demonstrated numerically in \Cref{sec:experiments}. These variants were motivated by the goal of reducing the overhead of the optimizer per-step beyond the baseline {MUDp} update; none reduced the wall-clock time to a target validation loss. 

\textbf{Hyperparameters:} 
For general hyperparameters, the experiments largely adopt those of Muon optimizers. In particular, the weight decay and unified learning rate proposed for Muon in \citet{liu2025muon_scalable} are used, eliminating the separate AdamW and Muon learning rates in favor of a single choice applied (appropriately scaled) to both subsets of parameters. The hyperparameter selection transfers well between Muon and MUD. The primary hyperparameter of MUD is the number of passes. The default of one pass yields the fastest wall-clock time to a given perplexity in most tested cases, with additional passes increasing wall-clock time for typically marginal improvement in convergence of loss. \Cref{sec:experiments} demonstrates examples where two passes yield improved convergence worth the additional computational cost (still $\ll$ Muon); these are cases where standalone AdamW struggles relative to Muon and MUD, indicating a more challenging optimization landscape that benefits from further decorrelation of matrix parameter directions. No instance has been found in which three MUD iterations improve wall-clock time to a given perplexity. A lighter Muon baseline was also tested by reducing the number of Newton--Schulz sweeps (e.g., three iterations instead of the default five). In the single-GPU settings considered, this modification did not yield a meaningful wall-clock speedup relative to the standard five-iteration Muon configuration and often degraded optimization behavior at matched hyperparameters. In some cases, only three Newton-Schulz iterations yielded instabilities in training dynamics not observed with MUD or Muon with five iterations. To provide a strong and representative Muon baseline, all reported results use the default Muon configuration with five Newton-Schulz iterations.

\textbf{Periodic recomputation of a cached operator:}
A distinct property of the triangular-solve-based decorrelation underlying MUD is that it can be viewed as applying a (data-dependent) linear operator to the current momentum direction. In principle, this suggests an amortization strategy: recompute a factor (e.g., triangular factor $T$ or its inverse $T^{-1}$) once every $k$ iterations and apply the cached operator on intermediate iterations, thereby reducing the average cost per step. Multiple variants of this idea (including fixed-period and smoothly increasing schedules for $k$) were implemented and tested in the context of LLM benchmarks in \Cref{sec:experiments}. Across a range of values $k$, no net improvement in wall-clock time was observed to a fixed validation perplexity. Although caching reduced the raw per-step cost in some regimes, the resulting degradation in convergence rate was comparable to (and often larger than) the computational savings, yielding no improvement in perplexity with respect to wall-clock time. 

\section{Numerical experiments}
\label{sec:experiments}

MUD is demonstrated on learning tasks for which Muon is known to be effective. The main results, reported in \Cref{sec:experiments:llms}, cover LLMs in multiple datasets, architectures, and GPUs. A short study on CIFAR-10 speedrunning is presented in \Cref{sec:experiments:cifar}, where the Newton-Schulz overhead is found to be marginal and Muon outperforms MUD in that regime. Last, we demonstrate the efficacy of MUD on a protein language model in \Cref{sec:experiments:esm2_scratch}.

\subsection{Transformer LLMs}
\label{sec:experiments:llms}

The setting is a prediction of the next-token. Given a contiguous stream of integer token IDs, the training objective is standard autoregressive cross-entropy: for token sequence $(x_1,\dots,x_T)$ the model predicts $x_{t+1}$ from the prefix $(x_1,\dots,x_t)$. Three language-modeling corpora differing in source and curation are considered, demonstrating that MUD's wall-clock advantages persist across distinct data distributions. \textbf{WikiText-103} consists of curated Wikipedia articles and is a long-standing benchmark for language modeling; its text is comparatively clean and structured, with long-range dependencies preserved at the document level~\cite{merity2017wikitext}.
\textbf{OpenWebText} (used in many nanoGPT-style experiments) is a WebText proxy built from URLs shared on Reddit, producing a broad but noisier mixture of web genres (blogs, forums, news, tutorials) with topical and stylistic biases induced by the collection pipeline~\cite{gokaslan2019openwebtext}.
Finally, \textbf{FineWeb-Edu} is an educationally filtered subset of the FineWeb corpus derived from Common Crawl on a modern pretraining scale; it is substantially larger and more aggressively filtered/ duplicated than the other benchmarks and is intended to better match contemporary LLM pretraining distributions~\cite{penedo2024fineweb}. All datasets use the same GPT-2 BPE tokenization and a single training script. Validation \emph{ loss} is the average negative logarithmic likelihood (cross-entropy) per token in data that are not available; the corresponding \emph{perplexity} $\mathrm{ppl}=\exp(\mathrm{val\_loss})$ is the conventional scale for language modeling performance. Computational efficiency is quantified through (i) throughput in tokens per second (\texttt{toks/s}), measured during training iterations, and (ii) elapsed wall-clock time to reach a target validation perplexity.

For Muon and MUD, a GPT-style parameter partitioning is used: token and position embeddings and the final language-model head are treated as other parameters and updated using AdamW, while all remaining matrix-shaped parameters (i.e., parameters with $\mathrm{ndim}\ge 2$) are treated as hidden parameters and receive the matrix optimizer update through Muon or MUD. AdamW applies to all parameters in the AdamW baseline, while Muon and MUD apply their matrix updates only to the hidden parameters and use AdamW on the remaining parameters internally. Optimization algorithms are evaluated in small, medium and large GPT-2 model configurations as described in \Cref{tab:gpt2-sizes}. The sequence length (context) is controlled by \texttt{block\_size} and is varied between experiments. 
\begin{table}[!h]
\centering
\small
\begin{tabular}{lcccc}
\toprule
Model & Layers & Heads & $d_{\text{model}}$ & Params \\
\midrule
GPT-2 small  & 12 & 12 & 768  & $\approx$125M \\
GPT-2 medium & 24 & 16 & 1024 & $\approx$355M \\
GPT-2 large  & 36 & 20 & 1280 & $\approx$775M \\
\bottomrule
\end{tabular}
\caption{GPT-2 model configurations used in the experiments (decoder-only Transformers).}
\label{tab:gpt2-sizes}
\end{table}

Unless otherwise specified, the following optimizer and training defaults are held fixed:
\begin{align*}
    &\texttt{warmup\_steps}=500,\;
    \texttt{dropout}=0.0,\;
    \texttt{lr}=10^{-3},\;
    \texttt{min\_lr}=0.1\cdot \texttt{lr} ,\\
    &\texttt{weight\_decay}=10^{-2},\;
    (\beta_1,\beta_2)=(0.9,0.95),\;
    \beta_{\text{Muon/MUD}}=0.95,\;
    \texttt{clip\_grad}=1.0,
\end{align*}
with a cosine learning rate that begins after a linear warmup. The learning rate decays to a nonzero floor given by a fixed fraction of the peak learning rate, i.e., $\texttt{min\_lr}=0.1\cdot \texttt{lr}$. The Muon/MUD learning rates are scaled as proposed in \citet{liu2025muon_scalable}, which has been found to be robust and comparable or superior to tuning an independent matrix-parameter learning rate. Reasonable tuning of the learning rate across all methods identifies the fixed learning rate $10^{-3}$ as consistently best across a range of problem and optimization configurations (with exceptions discussed in the following sections). All tests are run with \texttt{torch.compile} enabled and bf16 autocast. A linear warmup over \texttt{warmup\_steps} $= 500$ iterations is followed by a cosine decay schedule for the remainder of the training. 

The experiments were performed on three single-GPU systems: NVIDIA A100 (PCIe 40GB), AMD Instinct MI250, and NVIDIA Grace Hopper (GH200). The device is specified for each reported result. Unless otherwise stated, all comparisons within a figure or table are performed on the same hardware. Three random seeds ($[1203,3721,7865]$) are used for all results. Due to the noisy nature of stochastic optimization, a rolling mean window of length 7 is applied to statistics and speedups in all plots.

\subsubsection{Throughput}
\label{sec:experiments:throughput}

Before comparing time-to-quality, it is informative to examine the {peak training throughput} (tokens/s) achieved by AdamW, Muon, and MUD1 in representative hardware and training regimes. Throughput is configuration-dependent and has high variation across changes in batch size $B$, context length $T$ (block size), model size, and GPU architecture, all of which alter how much of each step is dominated by forward/backward compute versus optimizer overhead.

Tables~\ref{tab:throughput-gpt2small-nanogpt}, \ref{tab:throughput-gpt2medium-wikitext}, and \ref{tab:throughput-gpt2large-fineweb} report representative peak training performance (tokens/s) together with AdamW and MUD1 speed up relative to Muon, on three GPUs and several settings $(B,T)$. Each table uses a single dataset for clarity (NanoGPT/OpenWebText for GPT-2 small, WikiText-103 for GPT-2 medium, and FineWeb-Edu for GPT-2 large). Across datasets, throughput differences are generally small relative to the optimizer effects reported here; a notable exception is GPT-2 small on WikiText at $(B,T)=(8,1024)$, attributed to data/loader effects and the use of \texttt{int32} tokens (rather than \texttt{uint16} as in the other two corpora), which can modestly affect CPU-to-GPU transfer and batch construction overhead.

\begin{table}[!ht]
\centering
\small
\begin{tabular}{llrrrrrr}
\toprule
GPU & $(B,T)$ & AdamW & Muon & MUD1 & $\frac{\text{MUD}}{\text{Muon}}$ & $\frac{\text{AdamW}}{\text{Muon}}$ \\
\midrule
A100 & $(8,1024)$  & \num{131800} & \num{46900} & \num{87000}  & 1.85 & 2.81 \\
A100 & $(12,2048)$ & \num{130000} & \num{80300} & \num{108000} & 1.35 & 1.62 \\
A100 & $(8,4096)$  & \num{125700} & \num{88000} & \num{113700} & 1.29 & 1.43 \\
\midrule
GH200 & $(8,1024)$  & \num{342400} & \num{109000} & \num{185800} & 1.70 & 3.14 \\
GH200 & $(12,2048)$ & \num{391100} & \num{214600} & \num{295300} & 1.38 & 1.82 \\
\midrule
MI250 & $(8,1024)$  & \num{28200} & \num{14900} & \num{23000}  & 1.54 & 1.89 \\
MI250 & $(12,2048)$ & \num{73700} & \num{64000} & \num{68000}  & 1.06 & 1.15 \\
MI250 & $(8,4096)$  & \num{64900} & \num{57600} & \num{62400}  & 1.08 & 1.13 \\
\bottomrule
\end{tabular}
\caption{Training throughput (tokens/s) for GPT-2 small on the NanoGPT/OpenWebText benchmark. Absolute throughput and speedups relative to Muon are reported.}
\label{tab:throughput-gpt2small-nanogpt}
\end{table}

\begin{table}[!ht]
\centering
\small
\begin{tabular}{llrrrrrr}
\toprule
GPU & $(B,T)$ & AdamW & Muon & MUD1 & $\frac{\text{MUD}}{\text{Muon}}$ & $\frac{\text{AdamW}}{\text{Muon}}$ \\
\midrule
A100 & $(8,1024)$  & \num{49300} & \num{12600} & \num{29100} & 2.31 & 3.91 \\
A100 & $(12,2048)$ & \num{55200} & \num{26400} & \num{43450} & 1.65 & 2.09 \\
\midrule
GH200 & $(8,1024)$  & \num{142900} & \num{33600} & \num{68500}  & 2.04 & 4.25 \\
GH200 & $(12,2048)$ & \num{159000} & \num{72400} & \num{114600} & 1.58 & 2.20 \\
GH200 & $(12,4096)$ & \num{134000} & \num{89400} & \num{116200} & 1.30 & 1.50 \\
\midrule
MI250 & $(8,1024)$  & \num{28100} & \num{14900} & \num{23000} & 1.54 & 1.89 \\
MI250 & $(12,2048)$ & \num{30000} & \num{22700} & \num{27800} & 1.22 & 1.32 \\
\bottomrule
\end{tabular}
\caption{Training throughput (tokens/s) for GPT-2 medium on WikiText-103 (GPT-2 BPE tokens). Absolute throughput and speedups relative to Muon are reported.}
\label{tab:throughput-gpt2medium-wikitext}
\end{table}

\begin{table}[!ht]
\centering
\small
\begin{tabular}{llrrrrrr}
\toprule
GPU & $(B,T)$ & AdamW & Muon & MUD1 & $\frac{\text{MUD}}{\text{Muon}}$ & $\frac{\text{AdamW}}{\text{Muon}}$ \\
\midrule
A100  & $(8,1024)$ & \num{25000} & \num{4500}  & \num{13200} & 2.93 & 5.56 \\
MI250 & $(8,1024)$ & \num{13900} & \num{6200}  & \num{10900} & 1.76 & 2.24 \\
\midrule
GH200 & $(8,1024)$ & \num{70700} & \num{12000} & \num{30600} & 2.55 & 5.89 \\
GH200 & $(8,2048)$ & \num{71600} & \num{20900} & \num{44500} & 2.13 & 3.43 \\
GH200 & $(8,4096)$ & \num{66400} & \num{31400} & \num{52200} & 1.66 & 2.11 \\
\bottomrule
\end{tabular}
\caption{Training throughput (tokens/s) for GPT-2 large on FineWeb-Edu (GPT-2 BPE tokens). Absolute throughput and speedups relative to Muon
are reported.}
\label{tab:throughput-gpt2large-fineweb}
\end{table}

Across all model sizes, configurations, and GPUs, AdamW provides the highest raw tokens/s, as expected from its low optimizer overhead. Muon is consistently slowest per step due to the cost of repeated polar-approximation iterations, while MUD1 sits in between and is often substantially closer to AdamW than Muon. Muon's overhead is most pronounced in (i) regimes with smaller context length or smaller per step computation, where optimizer cost is a larger fraction of step time, and (ii) larger models (e.g., GPT-2 large vs.\ GPT-2 small), where Muon performs many large GEMMs per step across many weight matrices. As context length increases, all methods become more compute-dense, and relative throughput gaps shrink, consistent with optimizer overhead being increasingly amortized by forward/backward compute. Nevertheless, MUD1 consistently narrows the gap to AdamW by replacing Muon's multi-step polar update with a one-pass triangular-solve decorrelation, yielding 1.3--2.6$\times$ higher throughput than Muon in most settings and up to nearly $3\times$ on GPT-2 large.

\subsubsection{NanoGPT-style language modeling}
\label{sec:experiments:nanogpt}

We first evaluate AdamW, Muon, and MUD on OpenWebText (OWT), the primary full-scale corpus of the NanoGPT repository (\url{https://github.com/karpathy/nanoGPT/tree/master}) to reproduce the GPT-2 results. OpenWebText is an open-source recreation of OpenAI’s WebText: it collects links to Reddit posts above a small upvote threshold, aiming to approximate the diversity and quality of the original GPT-2 training data. Following NanoGPT’s standard preprocessing, OWT is tokenized with the GPT-2 BPE tokenizer and serialized into contiguous binary token streams (train.bin, val.bin) for efficient memory-mapped loading during training; the prepared split is roughly ~9B training tokens with a small holdout validation set.

\begin{figure}[!htb]
\centering
\includegraphics[width=0.425\linewidth]{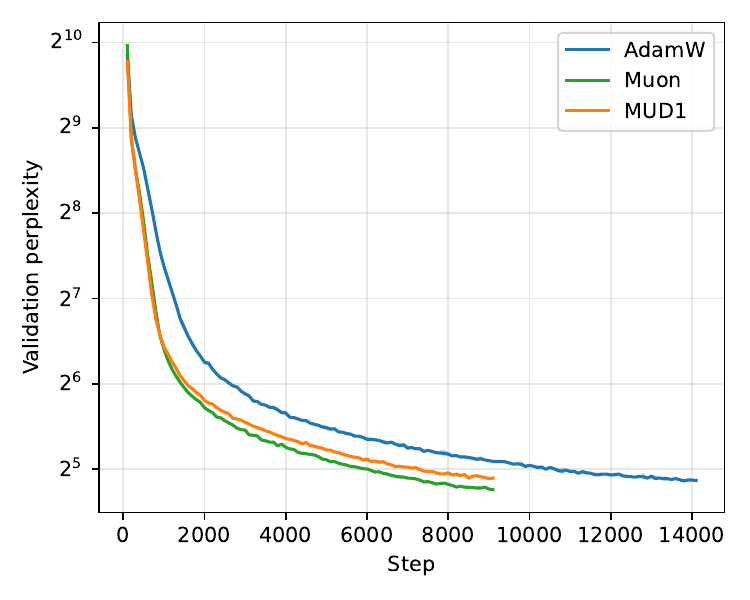}
\includegraphics[width=0.425\linewidth]{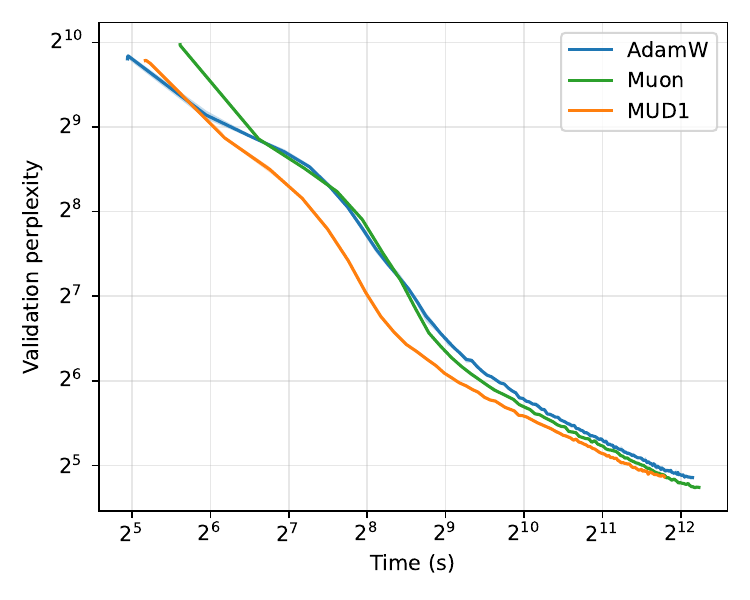}
\caption{Nvidia GH200 (\texttt{bs}=12,\ \texttt{T}=2048, GPT2-Medium): validation perplexity vs.\ step and vs.\ wall-clock time for AdamW, Muon, and MUD.}
\label{fig:nanoMed-ppl}
\end{figure}

\textbf{GPT2-Medium:} The GPT2-Medium architecture is trained with batch size 12 and blocksize 2048, yielding $\approx25$K tokens per step, on a GH200 chip. Muon and MUD1 are trained to 6000 steps and AdamW to 9000 steps to reach comparable validation perplexity. Validation perplexity as a function of step and training time (excluding evaluation time) is shown in \Cref{fig:nanoMed-ppl}. Muon and MUD1 decrease validation perplexity significantly faster early in training than AdamW. In terms of wall-clock time, Muon offers only marginal improvements over AdamW until smaller validation perplexities; MUD1 provides meaningful accelerations in all observed regimes, particularly in the earlier part of training. \Cref{fig:nanoMed-speedup} plots speedup relative to AdamW and the improvement rate in wall-clock time relative to AdamW versus validation perplexity. MUD1 yields 20--35\% speedups over AdamW and comparable speedups over Muon until smaller validation perplexities, where Muon's improved orthogonalization yields improved convergence. At smaller validation perplexities, the improvement rate in wall-clock time relative to AdamW decreases. This is consistent with decorrelation-style matrix updates providing their largest practical benefit in the more nonstationary / high-loss phase, while later training becomes increasingly constrained by the intrinsic optimization difficulty of the remaining error, although MUD1 and Muon still produce a validation perplexity reduction 5--10\% faster than AdamW.

\begin{figure}[!hbt]
\centering
\includegraphics[width=0.425\linewidth]{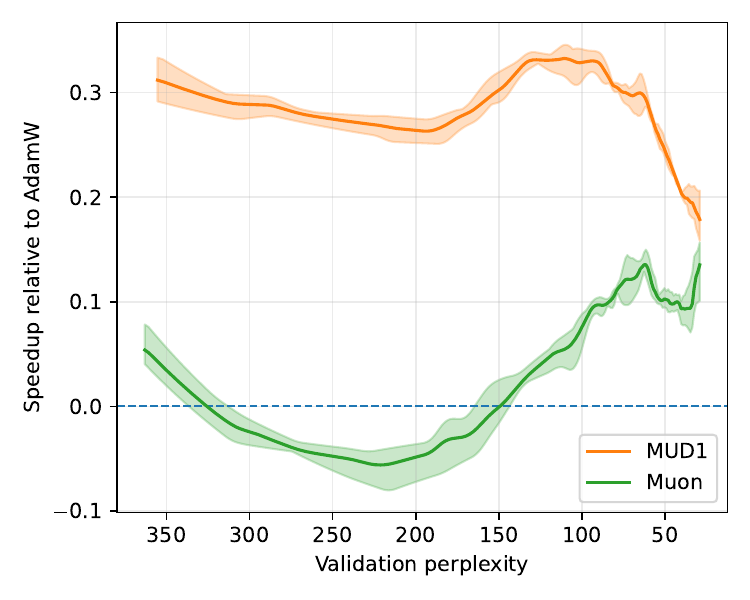}
\includegraphics[width=0.425\linewidth]{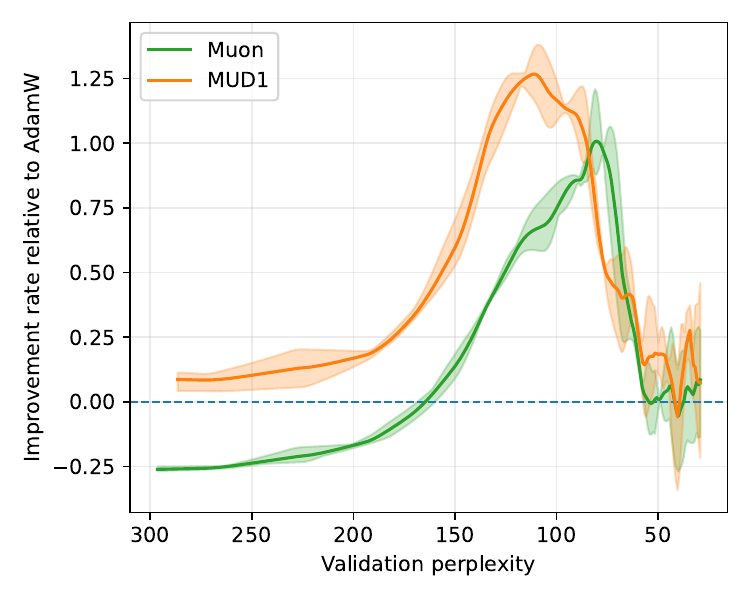}
\caption{Nvidia GH200 (\texttt{bs}=12,\ \texttt{T}=2048, GPT2-Medium): (left) speedup relative to AdamW vs.\ achieved validation perplexity; (right) relative perplexity improvement rate vs.\ perplexity, showing diminishing relative gains at low perplexity.}
\label{fig:nanoMed-speedup}
\end{figure}

\textbf{GPT2-Small:} A different configuration is now considered: the less expressive GPT2-Small architecture with batch size 24 and block size 4096, still on a GH200 chip. This yields almost 100K tokens per step, making the optimization landscape stiffer and more difficult to resolve. Repeated tests demonstrate that standalone AdamW is unstable above \texttt{lr\_adamw} $=3\cdot 10^{-4}$, so this modified learning rate is used for AdamW. Muon and MUD1 remain stable at the default higher learning rate $10^{-3}$. In contrast to the GPT2-Medium results (\Cref{fig:nanoMed-ppl,fig:nanoMed-speedup}), MUD1 produces noticeably worse convergence than Muon in smaller validation perplexities, motivating the inclusion of MUD2. \Cref{fig:nanoSmall} shows the speedup of Muon, MUD1, and MUD2 relative to AdamW and the validation perplexity versus the wall-clock time restricted to later training dynamics. Muon and both MUD variants yield 35--50\% speedups over AdamW across all observed regimes. MUD1 is fastest early in training, but as validation perplexity drops below 50, MUD1 performance quickly trails behind MUD2 and Muon, whose convergence remains more robust, and Muon performs best at smaller validation perplexities. This is apparent in the right plot, where Muon decreases the validation perplexity faster than all other methods, starting around perplexity 35. 

\begin{figure}[!htb]
\centering
\includegraphics[width=0.425\linewidth]{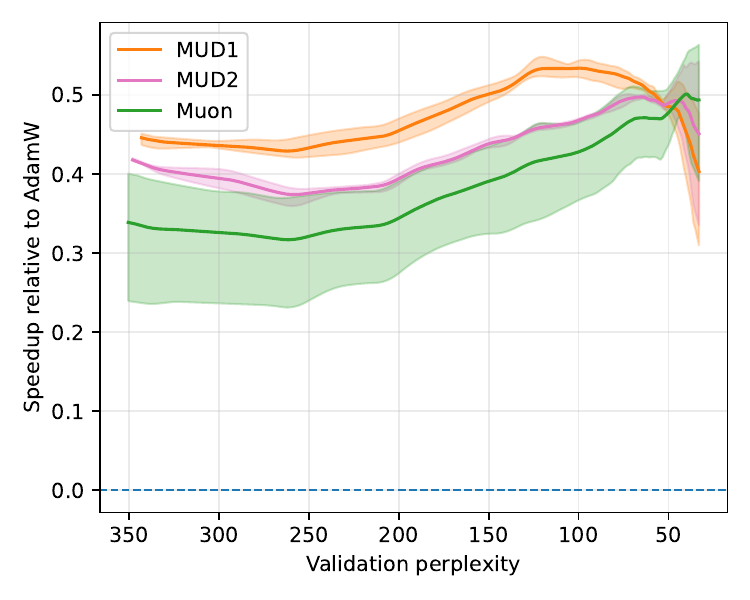}
\includegraphics[width=0.425\linewidth]{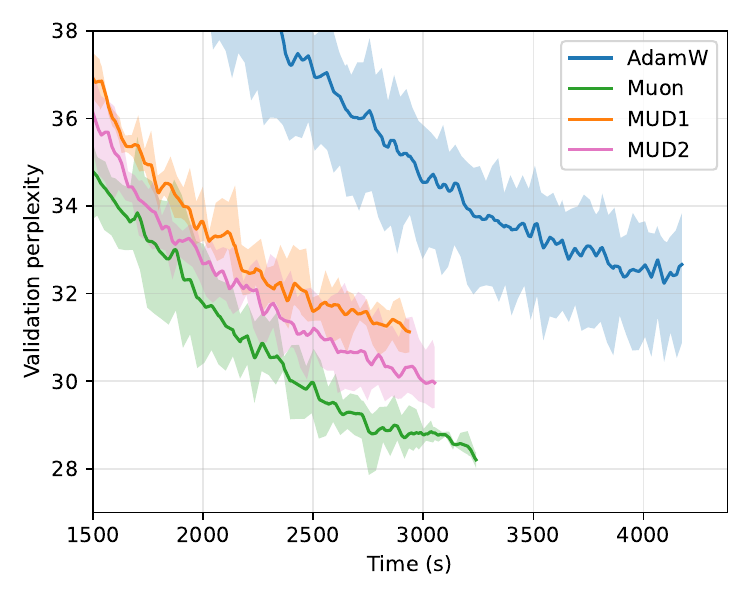}
\caption{Nvidia GH200 (\texttt{bs}=12,\ \texttt{T}=2048, GPT2-Small): validation perplexity vs.\ step and vs.\ wall-clock time for AdamW, Muon, and MUD.}
\label{fig:nanoSmall}
\end{figure}

\subsubsection{WikiText-103}
\label{sec:experiments:wikitext}

Training is now performed on the WikiText-103 dataset, preprocessed into GPT-2 BPE token streams stored as \texttt{train.bin} and \texttt{val.bin}. The learning problem and metrics are identical to the NanoGPT-style setting, but the data distribution differs: WikiText consists of curated Wikipedia articles with distinct structure and repetition patterns. Training throughput (\texttt{toks/s}) and wall-clock time to reach matched validation perplexity are reported to demonstrate that the observed speedups persist under a different data distribution while holding model architecture and optimizer hyperparameters fixed. The GPT2-Medium architecture is fixed, and GPU and blocksize are varied.

\begin{figure}[!htb]
\centering
\includegraphics[width=0.425\linewidth]{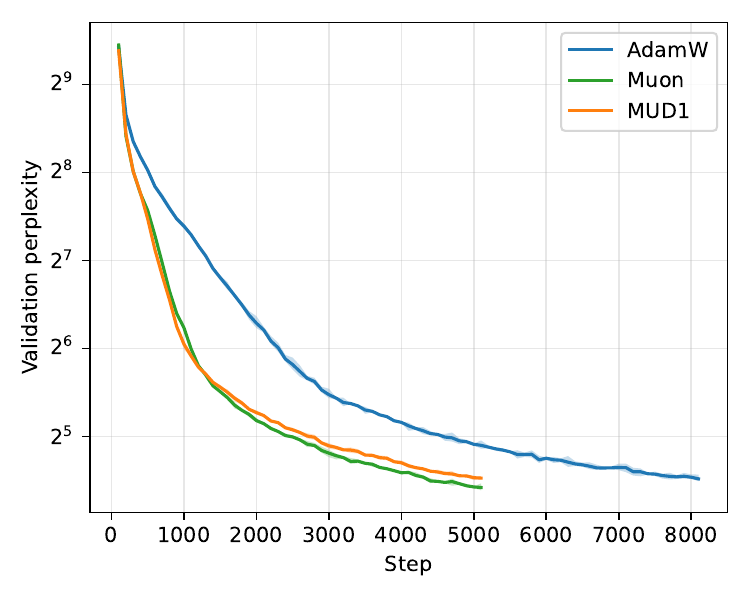}
\includegraphics[width=0.425\linewidth]{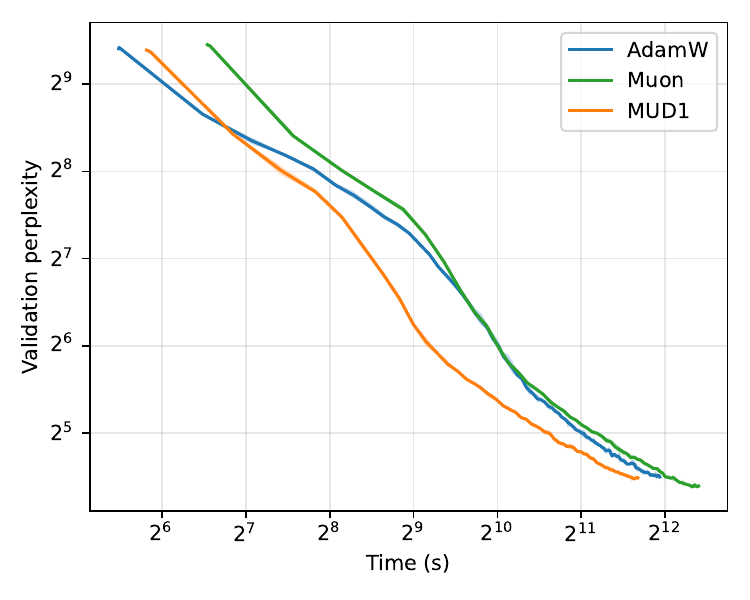}
\caption{Nvidia A100 (\texttt{bs}=12,\ \texttt{T}=2048, GPT2-Medium): validation perplexity vs.\ step and vs.\ wall-clock time for AdamW, Muon, and MUD.}
\label{fig:wikitext-a100-ppl}
\end{figure}

\textbf{NVIDIA A100 (PCIe 40GB):} Training is performed on an Nvidia A100 GPU with batch size 12 and block size 2048. \Cref{fig:wikitext-a100-ppl} shows validation perplexity as a function of both optimization steps and wall-clock time for the A100 tests. Although AdamW has the highest raw throughput, both Muon and MUD reduce perplexity substantially faster per step early and mid training. For this particular problem formulation, despite an effective reduction in perplexity, Muon does not appear to yield meaningful speedups over AdamW. In contrast, MUD achieves the best overall wall-clock time-to-perplexity among the three methods throughout the full tested training regime. The wall-clock efficiency is visualized in \Cref{fig:wikitext-a100-speedup}, which plots (i) the speedup of Muon and MUD relative to AdamW as a function of the achieved perplexity and (ii) the relative improvement rate of the perplexity versus the wall-clock time. On A100, the overhead computational cost of Muon yields a 20--30\% slowdown relative to AdamW early in training and at best matches AdamW in wall-clock time to reach a given validation perplexity. In contrast, MUD yields 20--45\% speedups over AdamW (and more over Muon) across all observed training regimes. Toward the slow tail of training, the improvement rate with respect to wall-clock time of all three methods becomes comparable, so extended training would likely result in slowly diminishing returns as early speedups are averaged over a longer total time. 

\begin{figure}[!htb]
\centering
\includegraphics[width=0.425\linewidth]{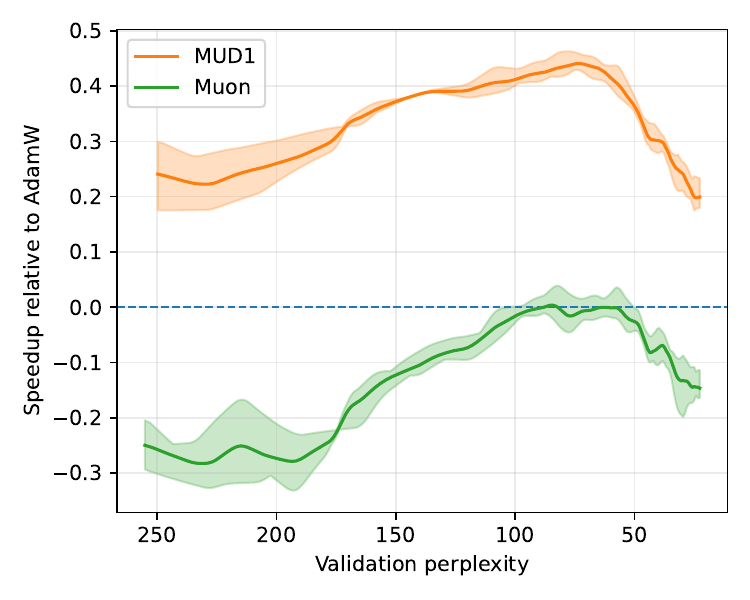}
\includegraphics[width=0.425\linewidth]{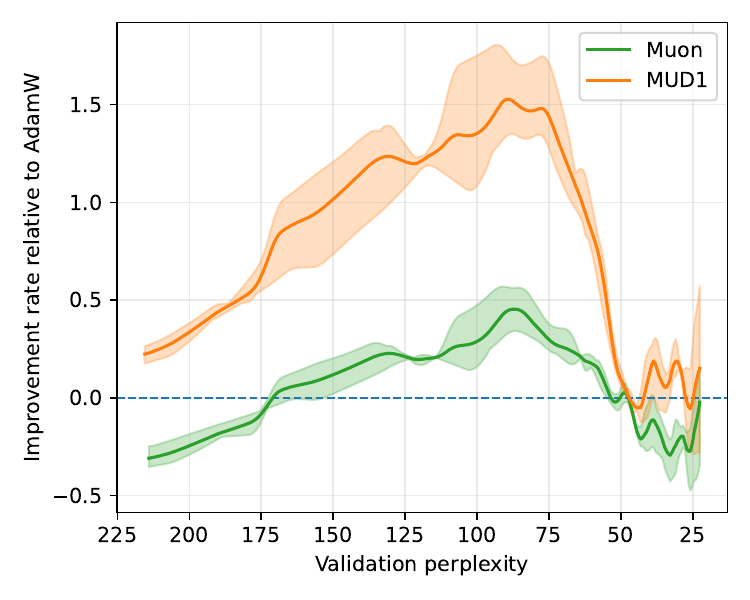}
\caption{Nvidia A100 (\texttt{bs}=12,\ \texttt{T}=2048, GPT2-Medium): (left) speedup relative to AdamW vs.\ achieved validation perplexity; (right) relative perplexity improvement rate vs.\ perplexity, showing diminishing relative gains at low perplexity.}
\label{fig:wikitext-a100-speedup}
\end{figure}

\textbf{Nvidia GH200:} Batch size 12 and large block size 4096 are used on the GH200 chip. Similar qualitative behavior is observed for MUD1, demonstrating significant speedups of 20--40\% to reach a target perplexity over AdamW and 10--25\% over Muon across the full range of training dynamics considered. As perplexity decreases (late training), the relative improvement rates of MUD/Muon compared to AdamW diminish, and the curves become closer. 

\begin{figure}[!htb]
\centering
\includegraphics[width=0.425\linewidth]{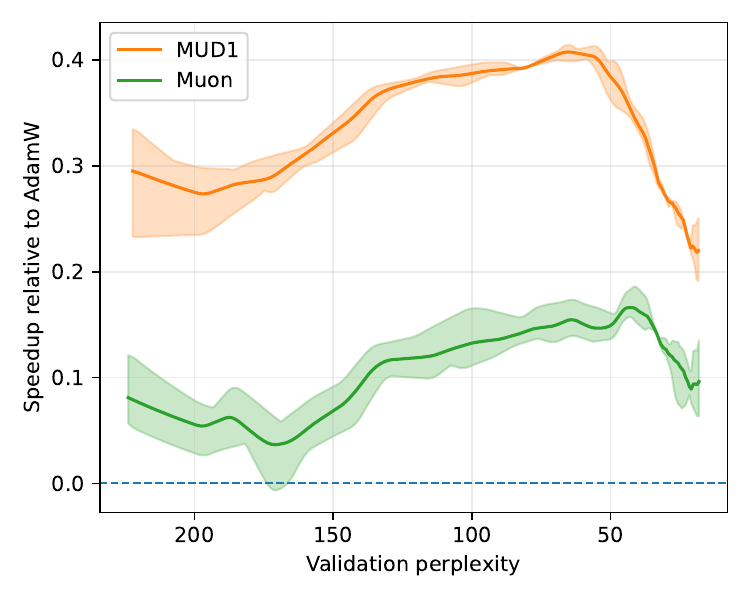}
\includegraphics[width=0.425\linewidth]{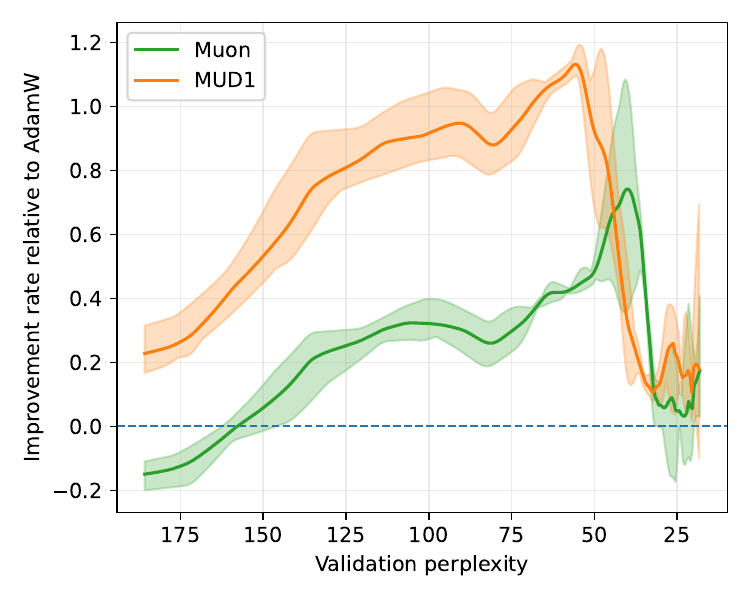}
\caption{Nvidia GH200 (\texttt{bs}=12,\ \texttt{T}=4096, GPT2-Medium): (left) speedup relative to AdamW vs.\ achieved validation perplexity; (right) relative perplexity improvement rate vs.\ perplexity, again showing diminishing relative gains at low perplexity.}
\label{fig:wikitext-gh-speedup}
\end{figure}

The speedup of MUD1 over AdamW is here comparable to that of the smaller blocksize on A100 (\Cref{fig:wikitext-a100-speedup}), but the performance of Muon is significantly better. This difference is unrelated to the optimization quality and is due solely to the improved Muon performance in this configuration. From the peak performance measurements in \Cref{tab:throughput-gpt2medium-wikitext}, the implied per-step overhead relative to AdamW varies substantially by GPU: on MI250, MUD incurs only $\approx 8.6\%$ overhead (vs.\ $\approx 32.8\%$ for Muon), while on A100 the overhead is much larger ($\approx 27.1\%$ for MUD and $\approx 109\%$ for Muon), which explains the relatively poor performance of Muon in \Cref{fig:wikitext-a100-speedup}. On GH200 at $T=4096$, the overhead is intermediate ($\approx 15.3\%$ for MUD and $\approx 49.8\%$ for Muon). 

\subsubsection{FineWeb-Edu}

To demonstrate that the wall-clock gains of MUD persist on a modern web-scale corpus with larger LLM models, models are trained on FineWeb-Edu, an educationally filtered subset of the FineWeb dataset~\cite{penedo2024fineweb}. A publicly released FineWeb-Edu preprocessing is used in a single concatenated GPT-2 BPE token stream, enabling the same training script and metrics as in the preceding experiments. Batch size 8 and block size 2048 are used on the GH200 chip (this does not fit in memory on the A100 or AMD MI250). AdamW is unstable at the learning rate $10^{-3}$ and converges slowly at $6\cdot 10^{-4}$; the rate $3\cdot 10^{-4}$ is found to be robust and is used for AdamW. Muon and MUD remain stable and perform best at $\approx 10^{-3}$. \Cref{fig:fineweb-large} shows the validation loss as a function of the step and the speedup relative to AdamW as a function of the validation loss (loss is used rather than perplexity here to provide an alternative performance metric) for MUD1, MUD2, Muon, and AdamW.

The overhead of Muon is substantial, achieving a given validation loss at least 60\% slower than AdamW in all cases. In contrast, MUD balances the optimization benefits of decorrelation with minimal computational overhead, providing a robust 10--20\% speedup over AdamW in all cases. MUD2 was also tested and provides minimal improvement in convergence, so the additional computational cost is not justified in this setting. This is the behavior observed in most experiments, with a few exceptions such as \Cref{fig:nanoSmall}.

\begin{figure}[!htb]
\centering
\includegraphics[width=0.425\linewidth]{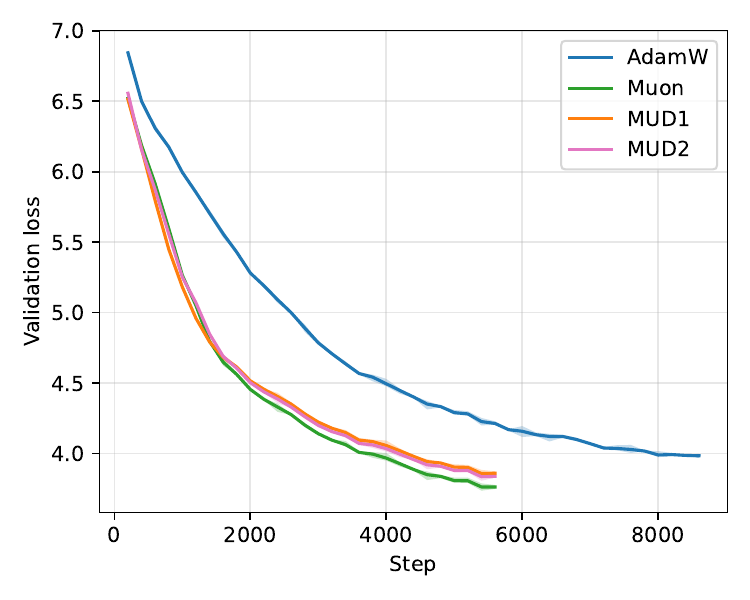}
\includegraphics[width=0.425\linewidth]{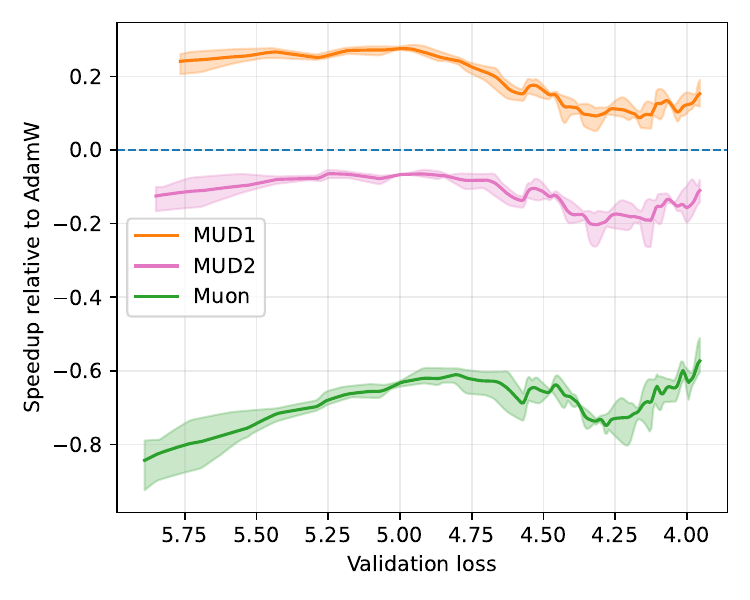}
\caption{Nvidia GH200 (\texttt{bs}=8,\ \texttt{T}=2048, GPT2-Large): (left) Validation loss as a function of step; (right) speedup relative to AdamW vs.\ achieved validation loss.}
\label{fig:fineweb-large}
\end{figure}

\subsection{CIFAR-10 speedrunning (AIRBench)}
\label{sec:experiments:cifar}

To identify regimes where MUD does not provide a large computational advantage over Muon, MUD is evaluated in the CIFAR-10 speedrunning setting using the public AIRBench Muon script associated with \citep{jordan202494}. The Muon speedrunning implementation is taken and the Newton--Schulz (NS) orthogonalization step is replaced with MUD decorrelation steps, while keeping the rest of the training recipe unchanged (augmentations, schedule, model, etc.). Two variants of MUD are reported, {MUD1} and {MUD3}. For Muon, NS is compared with three and five iterations ({Muon3} and {Muon5}). The results are shown in \Cref{tab:cifar-airbench}, averaged over 20 independent runs, reporting the mean and standard deviation of the accuracy of the final validation and the total run time. 

\begin{table}[!htb]
\centering
\small
\begin{tabular}{lcccc}
\toprule
Setting & Method & Val. Acc. (mean $\pm$ std) & Time [s] (mean $\pm$ std) \\
\midrule
\multirow{4}{*}{8 epochs (default)} 
& MUD1  & $0.9334 \pm 0.0014$ & $2.9795 \pm 0.0101$ \\
& MUD3  & $0.9339 \pm 0.0013$ & $3.2644 \pm 0.0111$ \\
& Muon3 & $0.9402 \pm 0.0011$ & $2.9690 \pm 0.0071$ \\
& Muon5 & $0.9403 \pm 0.0010$ & $3.0393 \pm 0.0085$ \\
\midrule
\multirow{4}{*}{13 epochs (+ momentum 0.7)} 
& MUD1  & $0.9407 \pm 0.0014$ & $4.7498 \pm 0.0178$ \\
& MUD3  & $0.9405 \pm 0.0011$ & $5.2078 \pm 0.0132$ \\
& Muon3 & $0.9445 \pm 0.0014$ & $4.7565 \pm 0.0157$ \\
& Muon5 & $0.9444 \pm 0.0010$ & $4.8239 \pm 0.0139$ \\
\bottomrule
\end{tabular}
\caption{CIFAR-10 speed running on AIRBench. Starting from the Muon speed running script, the Muon Newton-Schulz step is exchanged for MUD, using either MUD1 or MUD3. Muon baselines use 3 or 5 NS iterations. Results are mean $\pm$ std across 20 runs.}
\label{tab:cifar-airbench}
\end{table}

Across both settings, MUD does not retain a meaningful per-run wall-clock advantage over Muon, despite the theoretical complexity wins in \Cref{tab:muon-mud-cost}. In particular, MUD1 is only marginally faster than Muon5 in one configuration (8 epochs) and essentially tied in the other (13 epochs), while MUD3 is consistently slower than both Muon variants. The absence of a strong speed advantage is attributed to the convnet speedrunning regime: the overall step is not dominated by large dense matrix multiplications, and the GEMM-heavy NS updates for the relatively small matrices involved are highly optimized on modern GPUs compared to triangular-solve-based routines. In terms of accuracy, Muon achieves slightly better final validation accuracy than MUD in both settings, consistent with the trend observed in the Transformer/LLM experiments, where Muon can yield better optimization efficiency per step. Finally, increasing the strength of the orthogonalization/decorrelation step is not advantageous in this benchmark: moving from MUD1 to MUD3 increases wall-clock time without improving accuracy, and moving from Muon3 to Muon5 also provides negligible accuracy gains for a measurable time increase. Overall, this benchmark suggests that MUD is competitive in small vision speedrunning problems, but its primary advantages are more pronounced in large Transformer-style dense-matrix training regimes.

\subsection{ESM-2 150M training on omg\_prot50}
\label{sec:experiments:esm2_scratch}

This subsection examines the training of the ESM-2 150M protein language model on the \texttt{omg\_prot50} masked-language-modeling (MLM) task, which focuses on protein sequence modeling rather than natural-language text. The model used here is the 150M ESM-2 checkpoint with 30 transformer encoder layers, hidden size $H=640$, intermediate width $2560$, and 20 attention heads. It is referred to throughout as the 150M model, consistent with the standard public checkpoint naming and configuration.
The training stream is drawn from the entire \texttt{tattabio/ omega \_prot50} dataset in streaming mode, so no sequence is repeated within a 5{,}000-step run. The validation set consists of 2{,}500 sequences sampled once with a separate random seed before training begins, producing 312 evaluation batches.

We compare 5{,}000 training steps of AdamW, Muon, and two MUD variants on a single AMD Instinct MI300X. All runs use batch size 8, sequence length 512, 500 warmup steps, cosine decay to $10\%$ of the peak learning rate, and gradient clipping at 1.0. The highest learning rates are $\eta_{\mathrm{AdamW}} = 6 \times 10^{-4}$, $\eta_{\mathrm{Muon}} = 3\times 10^{-4}$, and $\eta_{\mathrm{MUD}} = 4 \times 10^{-3}$. Figure~\ref{fig:esm2-scratch-ppl} reports validation perplexity versus training step and wall-clock time on log--log axes.

After 5{,}000 steps, AdamW reaches its best validation perplexity of 17.18 in 466 seconds of wall-clock time, Muon reaches 16.84 in 997 seconds, and MUD1 and MUD2 reach 16.86 and 16.80 in 636 and 710 seconds, respectively. Thus, in this protein MLM task, Muon and both MUD variants achieve lower final validation perplexity than AdamW, with MUD matching Muon-level quality at substantially lower wall-clock cost. Unlike several of the preceding LLM experiments, the validation-perplexity gap between AdamW and Muon/MUD does not narrow late in training; instead, the advantage of Muon and MUD persists and narrows modestly over the 5{,} 2,000-step horizon.

\begin{figure}[!htb]
\centering
\includegraphics[width=0.90\linewidth]{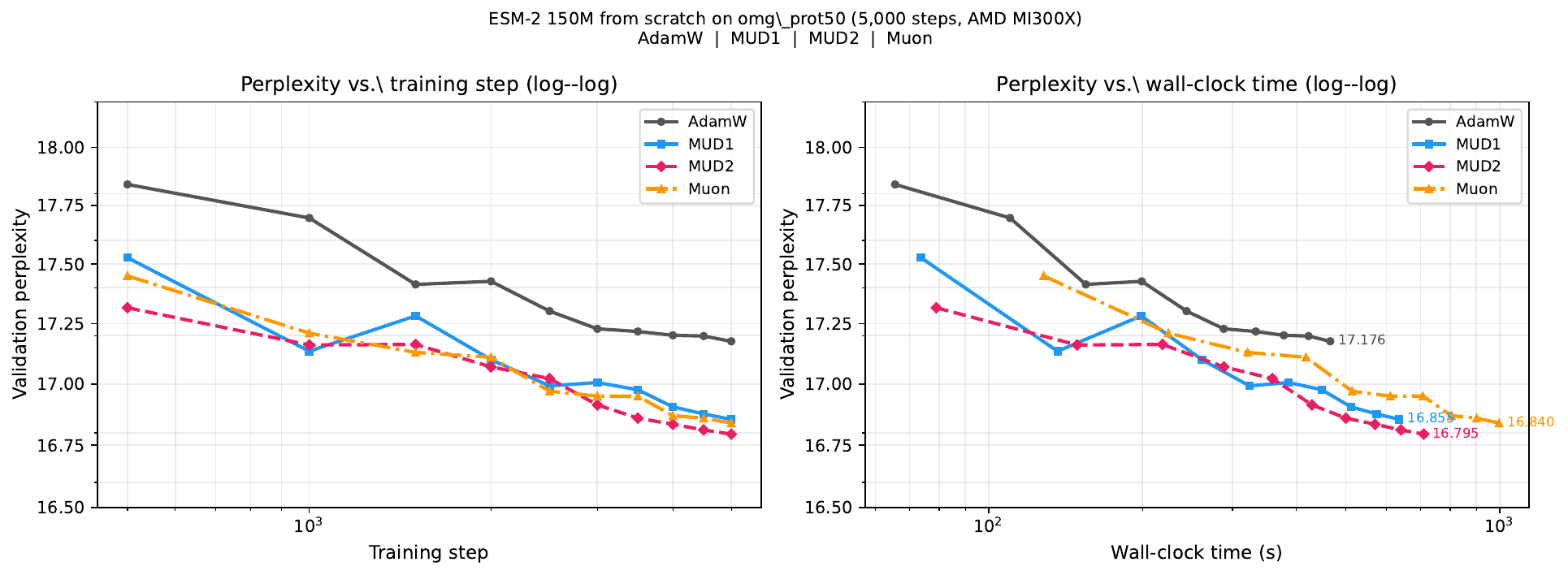}
\caption{ESM-2 150M training from scratch on omg\_prot50 (5{,}000 steps, AMD MI300X): validation perplexity vs.\ training step (left) and vs.\ wall-clock time (right), both on log--log axes, for AdamW ($\eta=6\times10^{-4}$), Muon ($\eta_{\mathrm{muon}}=3\times10^{-4}$), MUD1, and MUD2 ($\eta=4\times10^{-3}$). Final perplexities at step 4{,}999 are 17.176 (AdamW, 466s), 16.840 (Muon, 997s), 16.855 (MUD1, 636s), and 16.795 (MUD2, 710s). All four curves are still descending at the end of the training budget.}
\label{fig:esm2-scratch-ppl}
\end{figure}

\Cref{tab:esm2-scratch-ttt} reports wall-clock time to reach selected target validation perplexities, with speedups computed relative to AdamW via linear interpolation between evaluation points. MUD1 and MUD2 achieve $1.81-3.12\times$ speedups over AdamW at the evaluated targets, while Muon speedups are more modest at $1.1-1.66\times$. To that end, MUD1 is faster than Muon to every perplexity target in this experiment, eventually reaching comparable final quality in $36\%$ less total wall time. MUD1 and MUD2 are comparable to target perplexities, with MUD2 being slightly less noisy and achieving lower final validation perplexity. Unlike most LLM experiments, the additional cost of MUD2 is compensated by improved convergence in this setting. 

\begin{table}[!htb]
\centering
\small
\begin{tabular}{lcccc}
\toprule
Target ppl & AdamW & Muon & MUD1 & MUD2 \\
\midrule
17.5 & 141.2s & 127.8s (1.10$\times$) & \textbf{78.0s} (1.81$\times$) & 78.2s (1.81$\times$) \\
17.3 & 245.1s & 188.4s (1.30$\times$) & 110.0s (2.23$\times$) & \textbf{86.2s} (2.84$\times$) \\
17.2 & 392.8s & 236.7s (1.66$\times$) & \textbf{126.0s} (3.12$\times$) & 131.3s (2.99$\times$) \\
\midrule
Final ppl (step 4{,}999) & 17.176 & 16.840 & 16.855 & \textbf{16.795} \\
Total wall time          & 466s  & 997s  & 636s  & 710s \\
\bottomrule
\end{tabular}
\caption{ESM-2 150M training from scratch on omg\_prot50 (AMD MI300X, 5{,}000 steps): wall-clock time to first reach each target validation perplexity, with speedup relative to AdamW in parentheses. MUD1 and MUD2 reach every target faster than both AdamW and Muon.}
\label{tab:esm2-scratch-ttt}
\end{table}

\section{Conclusions}
\label{sec:conclusion}

We introduced \textbf{MUD} (MomentUm Decorrelation), a triangular solution-based whitening operator for matrix-shaped momentum updates in Transformer training. MUD provides an alternative numerical route to orthogonalized-update optimizers: instead of approximating a polar factor via repeated Newton--Schulz iterations (Muon), it applies a one- or few-pass approximate orthogonalization built from a lower-triangular approximation of the Gram matrix and a forward triangular solve. This preserves the matrix-aware motivation of orthogonalized momentum methods while substantially reducing per-step overhead in regimes where optimizer cost is a meaningful fraction of wall-clock training time.

On the theory side, fixed point properties are established, the inner Gram transform is related to symmetric Gauss--Seidel preconditioning, and a local contraction result is proved that yields quadratic convergence in vector-induced norms near the fixed point. On the empirical side, MUD produces consistent wall-clock improvements in time-to confusion relative to tuned AdamW and Muon in three text corpora (OpenWebText/NanoGPT, WikiText-103, FineWeb-Edu), three GPU platforms (A100, MI250, GH200), and GPT-2 small/medium / large architectures, including GPT-2 large (775M) where Muon’s optimizer overhead is particularly pronounced due to the quantity and size of linear weight matrices. Throughput measurements complement these findings: relative to Muon, MUD improves the maximum tokens/s by roughly $1.3$--$2.6\times$ in most configurations and up to nearly $3\times$ in GPT-2 large, demonstrating that a single triangular decorrelation pass can recover much of Muon’s optimization benefit at substantially lower cost. Finally, efficacy on other (non-text) transformer architectures is demonstrated by training the ESM-2 150M protein language model on the \texttt{omg\_prot50} masked-language-modeling task, where MUD matches the Muon-level validation perplexity with substantially reduced wall-clock time.

Future work will explore the composition of MUD with recent developments and advances in orthogonalization techniques for training transformers, such as those in \citet{si2025adamuon,lau2025polargrad,amsel2025polarexpress,grishina2025cans,boissin2025turbomuon}, as well as work on extending MUD to large-scale problems and distributed memory \citep{liu2025muon_scalable}. Results in this paper suggest that at smaller validation perplexities, the advantage of Muon and MUD over AdamW decreases, which is conceptually related to recent results demonstrating that quasi-Newton 2nd order optimizers are less effective at grokking than adaptive first order methods \citep{jiang2026convergence}. This motivates future work by training highly accurate models and fine-tuning regimes, but this may require training from scratch, as it was observed in \citet{liu2025muon_scalable} that AdamW-pre-trained models did not generalize/fine-tune well with Muon, and vice-versa.

\subsection*{Acknowledgements}
This work was funded in part by the National Nuclear Security Administration Interlab Laboratory Directed Research and Development program under project number 20250861ER. Los Alamos National Laboratory report number LA-UR-26-22012.  This research used resources provided by the Darwin testbed at Los Alamos National Laboratory (LANL) which is funded by the Computational Systems and Software Environments subprogram of LANL's Advanced Simulation and Computing program (NNSA/DOE). 

\bibliographystyle{plain}
\bibliography{refs.bib}

\clearpage
\appendix
\section{PyTorch implementation of MUD}
\begin{lstlisting}[language=Python]
"""
MUD (MomentUm Decorrelation) optimizer: a lightweight PyTorch implementation.

- Applies MUD to matrix-shaped parameters (ndim == 2): momentum + Nesterov lookahead
  followed by a multi-pass triangular-solve decorrelation (row-whitening surrogate).
- Applies AdamW to all other parameters (embeddings, heads, biases, layernorms, etc.).
"""
from __future__ import annotations
import math
from typing import Iterable, Tuple
import torch

@torch.no_grad()
def mud_whiten(M: torch.Tensor, passes: int = 1, eps: float = 1e-8) -> torch.Tensor:
    """ Multi-pass MUD transform on a 2D matrix M (n, m). """
    assert M.ndim == 2, "mud_whiten expects a 2D tensor"
    if passes < 1:
        raise ValueError(f"passes must be >= 1, got {passes}")

    n, m = M.shape
    transposed = n > m
    Q = M.t().contiguous() if transposed else M.contiguous()
    Q = Q.float() # fp32 for TRSM

    for _ in range(int(passes)):
        Q = Q / Q.norm(dim=1, keepdim=True).clamp_min(eps) # Row normalization
        G = Q @ Q.t() # Row Gram (k,k)
        T = torch.tril(G) # Lower-triangular of Gram
        Q = torch.linalg.solve_triangular(T, Q, upper=False) # Forward solve: T X = Q
        Q = Q / Q.norm(dim=1, keepdim=True).clamp_min(eps) # Renormalize rows

    if transposed:
        Q = Q.t().contiguous()

    return Q.to(dtype=M.dtype)

class MUDOptimizer(torch.optim.Optimizer):
    """MUD + AdamW hybrid optimizer."""
    def __init__(
        self,
        mud_params: Iterable[torch.nn.Parameter],
        adamw_params: Iterable[torch.nn.Parameter],
        lr: float = 1e-3,
        weight_decay: float = 1e-2,
        beta_mud: float = 0.95,
        adamw_betas: Tuple[float, float] = (0.9, 0.95),
        mud_passes: int = 1,
        eps: float = 1e-8,
    ):
        self._eps = float(eps)
        self._mud_passes = int(mud_passes)
        if self._mud_passes < 1:
            raise ValueError(f"mud_passes must be >= 1, got {mud_passes}")

        mud_params = list(mud_params)
        adamw_params = list(adamw_params)
        param_groups = [
            dict(params=mud_params, lr=lr, weight_decay=weight_decay, beta_mud=beta_mud, use_mud=True),
            dict(params=adamw_params, lr=lr, weight_decay=weight_decay, adamw_betas=adamw_betas, use_mud=False),
        ]
        defaults = dict(lr=lr, weight_decay=weight_decay)
        super().__init__(param_groups, defaults)

    @torch.no_grad()
    def step(self, closure=None):
        loss = None
        if closure is not None:
            with torch.enable_grad():
                loss = closure()

        for group in self.param_groups:
            if group.get("use_mud", False):
                self._mud_step(group)
            else:
                self._adamw_step(group)
        return loss

    def _mud_step(self, group: dict):
        lr = float(group["lr"])
        wd = float(group["weight_decay"])
        beta = float(group["beta_mud"])

        for p in group["params"]:
            if p.grad is None:
                continue
            g = p.grad
            state = self.state[p]
            if len(state) == 0:
                state["momentum"] = torch.zeros_like(p)

            p.mul_(1.0 - lr * wd) # decoupled weight decay
            m = state["momentum"]
            m.mul_(beta).add_(g) # momentum
            u = g + beta * m # Nesterov lookahead
            if p.ndim == 2:
                q = mud_whiten(u, passes=self._mud_passes, eps=self._eps)
                scale = 0.2 * math.sqrt(max(int(p.size(0)), int(p.size(1))))
                p.add_(q, alpha=-lr * scale)
            else:
                p.add_(u, alpha=-lr) # fallback for non-matrix params

    def _adamw_step(self, group: dict):
        lr = float(group["lr"])
        wd = float(group["weight_decay"])
        b1, b2 = group.get("adamw_betas", (0.9, 0.95))
        b1, b2 = float(b1), float(b2)
        eps = self._eps

        for p in group["params"]:
            if p.grad is None:
                continue
            g = p.grad
            state = self.state[p]
            if len(state) == 0:
                state["step"] = 0
                state["m"] = torch.zeros_like(p)
                state["v"] = torch.zeros_like(p)

            state["step"] += 1
            t = state["step"]
            p.mul_(1.0 - lr * wd) # decoupled weight decay
            m = state["m"]
            v = state["v"]
            m.mul_(b1).add_(g, alpha=1.0 - b1)
            v.mul_(b2).addcmul_(g, g, value=1.0 - b2)
            m_hat = m / (1.0 - b1 ** t)
            v_hat = v / (1.0 - b2 ** t)
            p.addcdiv_(m_hat, v_hat.sqrt().add_(eps), value=-lr)
\end{lstlisting}

\end{document}